\documentclass[letterpaper]{article} 
\usepackage{aaai24}  
\usepackage{times}  
\usepackage{helvet}  
\usepackage{courier}  
\usepackage[hyphens]{url}  
\usepackage{graphicx} 
\usepackage{natbib}  
\usepackage{caption} 
\usepackage{algorithm}

\usepackage{xspace}
    
\usepackage{amsmath,amsfonts,comment,dsfont}
\usepackage{algorithm}
\usepackage{algpseudocode}
\usepackage{graphicx}
\usepackage{amssymb}
\usepackage{subcaption}
\usepackage{amsthm}

\newtheorem{theorem}{Theorem}

\newtheorem{lemma}{Lemma}
\newtheorem{corollary}{Corollary}
\newtheorem{definition}{Definition}
\newtheorem{assumption}{Assumption}
\newtheorem{remark}{Remark}

\usepackage{algorithm}
\usepackage{algorithmicx}
\newcommand{\cN}{\mathcal{N}}
\newcommand{\cM}{\mathcal{M}}

\newcommand{\cS}{\mathcal{S}}

\newcommand{\cA}{\mathcal{A}}

\newcommand{\rmab}{\texttt{RMAB}\xspace}
\newcommand{\frmab}{\texttt{RMAB-F}\xspace}

\newcommand{\refrmab}{\texttt{Re-RMAB-F}\xspace}

\newcommand{\fair}{\texttt{FairRMAB}\xspace}
\newcommand{\fairucrl}{\texttt{Fair-UCRL}\xspace}

\newcommand{\gfairucrl}{\texttt{G-Fair-UCRL}\xspace}

 \usepackage{amsfonts}

 \usepackage{dsfont}   

\usepackage{xspace}

\usepackage{tikz}
\usetikzlibrary {positioning}

\definecolor {processblue}{cmyk}{0.96,0,0,0}

\usepackage{newfloat}
\usepackage{listings}
\DeclareCaptionStyle{ruled}{labelfont=normalfont,labelsep=colon,strut=off} 
\floatstyle{ruled}
\newfloat{listing}{tb}{lst}{}
\floatname{listing}{Listing}

\title{Online Restless Multi-Armed Bandits with Long-Term Fairness Constraints}
\author{
    Shufan Wang,
    Guojun Xiong,
    Jian Li
}
\affiliations{

    Stony Brook University\\
    
    \{shufan.wang, guojun.xiong, jian.li.3\}@stonybrook.edu
}

\begin{document}

\maketitle

\begin{abstract}
 
Restless multi-armed bandits (\rmab) have been widely used to model sequential decision making problems with constraints. The decision maker (DM) aims to maximize the expected total reward over an infinite horizon under an ``instantaneous activation constraint'' that at most $B$ arms can be activated at any decision epoch, where the state of each arm evolves stochastically according to a Markov decision process (MDP). However, this basic model fails to provide any fairness guarantee among arms. In this paper, we introduce \frmab, a new \rmab model with ``long-term fairness constraints'', where the objective now is to maximize the long-term reward while a minimum long-term activation fraction for each arm must be satisfied.  For the online \frmab setting (i.e., the underlying MDPs associated with each arm are unknown to the DM), 
we develop a novel reinforcement learning (RL) algorithm named \fairucrl.  We prove that \fairucrl ensures probabilistic sublinear bounds on both the reward regret and the fairness violation regret.  Compared with off-the-shelf RL methods, our \fairucrl is much more computationally efficient since it contains a novel exploitation that leverages a low-complexity index policy for making decisions.  Experimental results further demonstrate the effectiveness of our \fairucrl.

\end{abstract}

\section{Introduction}\label{sec:intro}

The restless multi-armed bandits (\rmab) model \cite{whittle1988restless} has been widely used to study sequential decision making problems with constraints, ranging from wireless scheduling \cite{sheng2014data,cohen2014restless}, 
resource allocation in general \cite{glazebrook2011general,larranaga2014index,borkar2017index}, 
to healthcare \cite{bhattacharya2018restless,mate2021risk,killian2021beyond}. In a basic \rmab setting, there is a collection of $N$ ``restless'' arms, each of which is endowed with a state that evolves independently according to a Markov decision process (MDP) \cite{puterman1994markov}. If the arm is activated at a decision epoch, then it evolves stochastically according to one transition kernel, otherwise according to a different transition kernel.  \rmab generalizes the Markovian multi-armed bandits \cite{lattimore2020bandit} by allowing arms that are not activated to change state, which leads to ``restless'' arms, and hence extends its applicability. For simplicity, we refer to a restless arm as an arm in the rest of the paper. Rewards are generated with each transition depending on whether the arm is activated or not. The goal of the decision maker (DM) is to maximize the expected total reward over an infinite horizon under an ``instantaneous activation constraint'' that at most $B$ arms can be activated at any decision epoch.

However, the basic \rmab model fails to provide any guarantee on how activation will be distributed among arms. This is also a salient design and ethical concern in practice, including mitigating data bias for healthcare \cite{mate2021risk,li2022efficient} and societal impacts \cite{yin2023long,biswas2023fairness}, providing quality of service guarantees to clients in network resource allocation \cite{li2019combinatorial}, just to name a few. In this paper, we introduce a new \textit{\rmab model with fairness constraints}, dubbed as \frmab to address fairness concerns in the basic \rmab model.  Specifically, we impose ``long-term fairness constraints'' into \rmab problems such that the DM must ensure a minimum long-term activation fraction for each arm \cite{li2019combinatorial,chen2020fair,d2020fairness,li2022efficient}, as motivated by aforementioned resource allocation and healthcare applications.  The DM's goal now is to maximize the long-term reward while satisfying not only ``\textit{instantaneous} activation constraint'' in \textit{each decision epoch} but also ``\textit{long-term} fairness constraint'' for \textit{each arm}.  Our objective is to develop \textit{low-complexity} reinforcement learning (RL) algorithms with \textit{order-of-optimal regret guarantees} to solve \frmab without knowing the underlying MDPs associated with each arm.

Though online \rmab has been gaining attentions, existing solutions cannot be directly applied to our online \frmab.  First, existing RL algorithms including state-of-the-art colored-UCRL2 \cite{ortner2012regret} and Thompson sampling methods \cite{jung2019regret,akbarzadeh2022learning}, suffer from an exponential computational complexity and regret bounds grow exponentially with the size of state space.  This is because those need to repeatedly solve Bellman equations with an exponentially large state space for making decisions. Second, though much effort has been devoted to developing low-complexity RL algorithms with order-of-optimal regret for online \rmab,   
many challenges remain unsolved.  For example, multi-timescale stochastic approximation algorithms \cite{fu2019towards,avrachenkov2022whittle} suffer from slow convergence and have no regret guarantee.  Adding to these limitations is the fact that none of them were designed with fairness constraints in mind, e.g., \cite{wang2020restless,xiong2022reinforcement,xiong2022learning,xiong2022reinforcementcache,xiong2023finite} only focused on minimizing costs in \rmab, while the DM in our \frmab faces a \textbf{new dilemma} on how to manage the balance between maximizing the \textit{long-term} reward and satisfying both \textit{instantaneous} activation constraint and \textit{long-term} fairness requirements. This adds a new layer of difficulty to designing low-complexity RL algorithms with order-of-optimal regret for \rmab that is already quite challenging.

To tackle this new dilemma, we develop \fairucrl, a novel RL algorithm for online \frmab. On one hand, we provide the first-ever regret analysis for online \frmab, and prove that \fairucrl ensures sublinear bounds (i.e., $\tilde{\mathcal{O}}(\sqrt{T})$) for both the reward regret (suboptimality of long-term rewards) and the fairness violation regret (suboptimality of long-term fairness violation) with high probability.  On the other hand, \fairucrl is computationally efficient.  This is due to the fact that \fairucrl contains a novel exploitation that leverages a low-complexity index policy for making decisions, which differs dramatically from aforementioned off-the-shelf RL algorithms that make decisions via solving complicated Bellman equations. Such an index policy in turn guarantees that the instantaneous activation constraint can be always satisfied in each decision epoch. To the best of our knowledge, \fairucrl is the first model-based RL algorithm that simultaneously provides (i) order-of-optimal regret guarantees on both the reward and fairness constraints; and (ii) a low computational complexity, for \frmab in the online setting. Finally, experimental results on real-world applications (resource allocation and healthcare) show that \fairucrl effectively guarantees fairness for each arm while ensures good regret performance.

\section{Model and Problem Formulation}\label{sec:model}
In this section, we provide a brief overview of the conventional \rmab, and then formally define our \frmab as well as the online settings considered in this paper.

\subsection{Restless Multi-Armed Bandits} 
A \rmab problem consists of a DM and $N$ arms \cite{whittle1988restless}.  Each arm $n\in\cN = \{1,...,N\}$ is described by a unichain MDP $M_n$ \cite{puterman1994markov}.  Without loss of generality (W.l.o.g.), all MDPs $\{M_n, \forall n\in\cN\}$ share the same finite state space $\cS$ and action space $\cA:=\{0,1\}$, but may have different transition kernels $P_n(s^\prime|s, a)$ and reward functions $r_n(s,a)$, $\forall s, s^\prime\in\cS, a\in\cA$.  Denote the cardinalities of  $\cS$ and $\cA$ as $S$ and $A$, respectively.  The initial state is chosen according to the initial state distribution $\boldsymbol{s}_0$ and $T$ is the time horizon.  
At each time/decision epoch $t$, the DM observes the state of each arm $n$, denoted by $s_n(t)$, and activates a subset of $B$ arms.  Arm $n$ is called \textit{active} when being activated, i.e., $a_n(t)=1$, and otherwise \textit{passive}, i.e., $a_n(t)=0.$  Each arm $n$ generates a stochastic reward $r_n(t):=r_n(s_n(t), a_n(t))$, depending on its state $s_n(t)$ and action $a_n(t)$.  W.l.o.g., we assume that $r_n\in[0,1]$ with mean $\bar{r}_n(s,a), \forall n, s, a$, and only active arms generate reward, i.e., $r_n(s,0)=0,\forall n, s$. 
Denote the sigma-algebra generated by random variables $\{(s_n(\tau), a_n(\tau)), \forall n, \tau<t\}$ as $\mathcal{F}_t$.  The goal of the DM is to design a control policy $\pi:\mathcal{F}_t\mapsto\cA^N$ to maximize the total expected reward, which can be expressed as  $\liminf_{T \rightarrow \infty}\frac{1}{T} \mathbb{E} \sum_{t=1}^T \sum_{n=1}^N r_n(t)$, under the  ``instantaneous activation constraint'', i.e., $\sum_{n=1}^N a_{n}(t) \leq B, \forall t$.

\subsection{\rmab with Long-Term Fairness Constraints}

In addition to maximizing the long-term reward, ensuring long-term fairness among arms is also important for real-world applications \cite{yin2023long}.  As motivated by applications in network resource allocation and healthcare \cite{li2019combinatorial,li2022efficient}, we impose a ``long-term fairness constraint'' on a minimum long-term activation fraction for each arm, i.e., $\liminf_{T \rightarrow \infty} \frac{1}{T} \mathbb{E}\sum_{t=1}^T a_n(t) \geq \eta_n, \forall n\in\cN$, where $\eta_n\in(0,1)$ indicates the minimum fraction of time that arm $n$ should be activated. To this end, the objective of \frmab is now to maximize the total expected reward while ensuring that both ``instantaneous activation constraint'' at each epoch and ``long-term fairness constraint'' for each arm are satisfied. Specifically, $\frmab(P_n, r_n, \forall n)$ is defined as:
\begin{align}
 \max_{\pi} ~&\liminf_{T \rightarrow \infty} \frac{1}{T} \mathbb{E}\Bigg[\sum_{t=1}^T \sum_{n=1}^N r_n(t) \Bigg] \displaybreak[0]\label{eq:FRMAB-obj}\\
    \text{s.t.}~&\sum_{n=1}^N a_{n}(t) \leq B, \quad \forall t, \displaybreak[1]\label{eq:FRMAB-constraint1}\\
    & \liminf_{T \rightarrow \infty} \frac{1}{T} \mathbb{E}\Bigg[\sum_{t=1}^T a_n(t)\Bigg] \geq \eta_n, \quad\forall n. \label{eq:FRMAB-constraint2}
\end{align}

\begin{assumption}
We assume that the \frmab problem of~(\ref{eq:FRMAB-obj})-(\ref{eq:FRMAB-constraint2}) is feasible, i.e., there exists a policy $\pi$ such that constraints~(\ref{eq:FRMAB-constraint1}) and~(\ref{eq:FRMAB-constraint2}) are satisfied.   
\end{assumption}

Note that in this paper, we only consider learning feasible \frmab by this assumption.  When the underlying MDPs (i.e., $P_n$ and $r_n$) associated with each arm $n\in\cN$ are known to the DM, we can compute the offline optimal policy $\pi^{opt}$ by treating the offline \frmab as an infinite-horizon average cost per stage problem using relative value iteration \citep{puterman1994markov}.  However, it is well known that this approach suffers from the curse of dimensionality due to the explosion of state space \cite{papadimitriou1994complexity}.

\subsection{Online Settings}

We focus on online \frmab, where the DM repeatedly interacts with N arms $\{M_n=\{\cS, \cA, P_n, r_n\}, \forall n\in\cN\}$ in an episodic manner.  Specifically, the time horizon $T$ is divided into $K$ episodes and each episode consists of $H$ consecutive frames, i.e., $T=KH$. The DM is not aware of the values of the transition kernel $P_n$ and reward function $r_n$, $\forall n\in\cN$. Instead, the DM estimates the transition kernels and reward functions in an online manner by observing the trajectories over episodes.  As a result, it is not possible for a learning algorithm to unconditionally guarantee constraint satisfaction in~(\ref{eq:FRMAB-constraint1}) and~(\ref{eq:FRMAB-constraint2}) over a finite number of episodes.  To this end, we measure the performance of a learning algorithm with policy $\pi$ using two types of \textit{regret}. 

First, the regret of a policy $\pi$ with respect to the long-term reward against the offline optimal policy 
$\pi^{opt}$ is defined as 
\begin{align}\label{eq:regret-reward}
\Delta_T^{R}:= T V^{\pi^{opt}}-\mathbb{E}_{\pi}\Bigg[\sum_{t=1}^T\sum_{n=1}^N r_n(t)\Bigg],
\end{align}
where $V^{\pi^{opt}}$ is the long-term reward obtained under the offline optimal policy  $\pi^{opt}$. Note that since finding $\pi^{opt}$ for \frmab is intractable, we characterize the regret with respect to a feasible, asymptotically optimal index policy (see Theorem \ref{thm:asy-optimality} in supplementary materials), similar to the regret definitions for online \rmab \cite{akbarzadeh2022learning,xiong2022learning}.

Second, the regret of a policy $\pi$ with respect to the long-term fairness against the minimum long-term activation fraction $\eta_n$ for each arm $n$, or simply the fairness violation is 
\begin{align}\label{eq:regret-fair}
\Delta_T^{n,F}:= T \eta_n-\mathbb{E}_{\pi}\Bigg[\sum_{t=1}^Ta_n(t)\Bigg],~\forall n\in\cN. 
\end{align}

\section{\fairucrl and Regret Analysis}\label{sec:learning}
We first show that it is possible to develop an RL algorithm for the computationally intractable \frmab problem of~(\ref{eq:FRMAB-obj})-(\ref{eq:FRMAB-constraint2}).   Specifically, we leverage the popular UCRL \citep{jaksch2010near} to online \frmab, and develop an episodic RL algorithm named \fairucrl. 
On one hand, \fairucrl strictly meets the ``instantaneous activation constraint''~(\ref{eq:FRMAB-constraint1}) at each decision epoch since it leverages a low-complexity index policy for making decisions at each decision epoch, and hence \fairucrl is computationally efficient. On the other hand, we prove that \fairucrl provides probabilistic sublinear bounds for both  reward regret and fairness violation regret. To our best knowledge, \fairucrl is the first model-based RL algorithm to provide such guarantees for online \frmab.

\begin{algorithm}[t]
\caption{\fairucrl}
\label{alg:UCRL}
\begin{algorithmic}[1]
	\State {\bfseries Require:} Initialize $C_n^{0}(s,a)=0,$ and  $\hat{P}_n^{0}(s^\prime|s,a)=1/S$, $\forall n\in\cN, s,s^\prime\in\cS, a\in\cA$.	
	\For{$k=1,2,\cdots,K$}
	\State $//** $\textit{Optimistic Planning}$ **//$
	\State Construct the set of plausible MDPs $\cM^k$ as in \eqref{eq:plausible21}; 
	\State Relaxed the instantaneous activation constraint in \frmab($\tilde{P}_n^k, \tilde{r}_n^k, \forall n$) to be ``long-term activation constraint'', and transform it into 
  $\textbf{ELP}(\mathcal{M}^k, z^k)$ in~\eqref{eq:ELP};
	\State $//** $\textit{Policy Execution}$ **//$
	\State Establish the \fair index policy $\pi^{k,*}$ on top of the solutions to the ELP and execute it. 
	\EndFor
\end{algorithmic}
\end{algorithm}

\subsection{The \fairucrl Algorithm}

\fairucrl proceeds in episodes as summarized in Algorithm~\ref{alg:UCRL}.  Let $\tau_k$ be the start time of episode $k$. \fairucrl maintains two counts for each arm $n$. Let $C_{n}^{k-1}(s,a)$ be the number of visits to state-action pairs $(s,a)$ until $\tau_k$, and $C_{n}^{k-1}(s,a, s^\prime)$ be the number of transitions from $s$ to $s^\prime$ under action $a$ until $\tau_k$.  
Each episode consists of two phases: 

\subsubsection{Optimistic planning.} 
At the beginning of each episode, \fairucrl constructs a confidence ball that contains a set of plausible MDPs \citep{ xiong2022learning} for each arm $\forall n\in\cN$ with high probability. The ``center'' of the confidence ball has the transition kernel and reward function that are computed by the corresponding empirical averages as: 
$\hat{P}_{n}^{k}(s^\prime|s,a)=\frac{C_{n}^{k-1}(s,a,s^\prime)}{\max\{C_{n}^{k-1}(s,a),1\}},$ $\hat{r}_n^{k}(s,a)=
\frac{\sum\limits_{l=1}^{k-1}\sum\limits_{h=1}^{H}r_n^l(s,a)\mathds{1}(s_n^{l}(h)= s, a_n^{l}(h)=a)}{\max\{C_n^{k-1}(s,a),1\}}. $

The ``radius'' of the confidence ball is set to be $\delta_n^{k}(s, a)$ according to the Hoeffding inequality. Hence the set of plausible MDPs in episode $k$ is:
\begin{align}\label{eq:plausible21}
      \mathcal{M}^k &=  \big\{M_n^k  = (\cS, \cA, \tilde{P}_n^k,  \tilde{r}_n^k): 
      |\tilde{P}_n^k(s^\prime|s, a) - \hat{P}_n^{k}(s^\prime|s, a)| \nonumber\\
      &\leq  \delta_n^{k}(s, a), 
      \tilde{r}_n^k(s, a) = \hat{r}_n^{k}(s, a) + \delta_n^{k}(s, a)\big\},
\end{align}

 \fairucrl then selects an optimistic MDP $M_n^k, \forall n$ and an optimistic policy with respect to \frmab($\tilde{P}_n^k, \tilde{r}_n^k, \forall n$).  Since solving \frmab($\tilde{P}_n^k, \tilde{r}_n^k, \forall n$) is intractable, we first relax the instantaneous activation constraint so as to achieve a ``long-term activation constraint'', i.e., the activation. 
 It turns out that this relaxed problem can be equivalently transformed into a linear programming (LP) via replacing all random variables in the relaxed \frmab($\tilde{P}_n^k, \tilde{r}_n^k, \forall n$) with the occupancy measure corresponding to each arm $n$ \citep{altman1999constrained}. 
Due to lack of knowledge of transition kernels and rewards, we further rewrite it as an extended LP (ELP) by leveraging \textit{state-action-state occupancy measure} $z_{n}^k(s, a, s^\prime)$ 
to express confidence intervals of transition probabilities: given a policy $\pi$ and transition functions $\tilde{P}_n^k,$ the occupancy measure $z_{n}^k(s, a, s^\prime)$ induced by $\pi$ and $\tilde{P}_n^k$ is that $\forall n, s,s^\prime, a, k$:
    $z_n^k(s,a, s^\prime)\!:=\! \lim_{H\rightarrow \infty}\frac{1}{H}\mathbb{E}_\pi [ \sum_{h=1}^{H-1} \mathds{1} (s_n(h)\!=\! s, a_n(h)\!=\!a, s_n(h+1)= s^\prime)].$
The goal is to solve the extended LP as 
\begin{align}\label{eq:ELP}
    z^{k,*}=\arg\min_{z^k} \textbf{ELP}(\mathcal{M}^k, z^k),
\end{align}
with $z^{k,*}:=\{z_{n}^{k,*}(s, a, s^\prime), \forall n\in\cN\}$. We present more details on ELP in supplementary materials.

\subsubsection{Policy execution.} We construct an index policy, which is feasible for the online \frmab($\tilde{P}_n^k, \tilde{r}_n^k, \forall n$) as inspired by \citet{xiong2022learning}.  Specifically, we derive our index policy on top of the optimal solution $z^{k,*}=\{z_n^{k,*}, \forall n\in\cN\}$.  Since $\cA=\{0,1\}$, i.e, an arm can be either active or passive at time $t$, we define the index assigned to arm $n$ in state $s_n(t)=s$ at time $t$ to be as
\begin{align}\label{eq:fair-index}
\omega_{n}^{k,*}(s):=\frac{\sum_{s^\prime}z_{n}^{k,*}(s,1,s^\prime)}{\sum_{a,s^\prime}z_{n}^{k,*}(s,a,s^\prime)}, \quad\forall  n \in\cN. 
\end{align}
We call this \textit{the fair index} and rank all arms according to their indices in~(\ref{eq:fair-index}) in a non-increasing order, and activate the set of $B$ highest indexed arms, denoted as $\cN(t)\subset\cN$ such that $\sum_{n\in\cN(t)} a_n^*(t)\leq B$. All remaining arms are kept passive at time $t$. We denote the resultant index-based policy, which we call the \fair index policy as $\pi^{k,*}:=\{\pi_{n}^{k,*}, \forall n\in\cN\}$, and execute this policy in this episode. More discussions on the property of the \fair index policy are provided in supplementary materials.

\begin{remark}\label{remark:alg1}
Although \fairucrl draws inspiration from the infinite-horizon UCRL 
\cite{jaksch2010near, xiong2022learning},  there exist a major difference. \fairucrl modifies the principle of optimism in the face of uncertainty for making decisions which is utilized by UCRL based algorithms, to not only maximize the long-term rewards but also to satisfy the long-term fairness constraint in our \frmab.  This difference is further exacerbated since the objective of conventional regret analysis, e.g., colored-UCRL2 \cite{ortner2012regret, xiong2022learning} for \rmab is to bound the reward regret, while due to the long-term fairness constraint, we also need to bound the fairness violation regret for each arm for \fairucrl, which will be discussed in details in Theorem~\ref{thm:regret}. We note that the designs of our \fairucrl and the \fair index policy are largely inspired by the LP based approach in \citet{xiong2022learning} for online \rmab.  However, \citet{xiong2022learning} only considered the instantaneous activation constraint, and hence is not able to address the new dilemma faced by our online \frmab, which also needs to ensure the long-term fairness constraints.   Finally, our \gfairucrl with no fairness violation further distinguishes our work.  
\end{remark}

\subsection{Regret Analysis of \fairucrl}\label{sec:regret}

We now present our main theoretical results on bounding the regrets defined in~(\ref{eq:regret-reward}) and~(\ref{eq:regret-fair}), realizable by \fairucrl. 
\begin{theorem}\label{thm:regret}
When the size of the confidence intervals $\delta_{n}^{k}(s,a)$ is built  for $\epsilon\in(0,1)$ as
\begin{align*} 
\delta_{n}^{k}(s,a)=\sqrt{\frac{1}{2C_{n}^{k}(s,a)}\log\bigg(\frac{SAN(k-1)H}{\epsilon}\bigg)},
\end{align*}
with probability at least $ 1-(\frac{\epsilon}{{SANT}})^{\frac{1}{2}}$,
\fairucrl achieves the reward regret as:
\begin{align*}
\Delta_T^{R}=\tilde{\mathcal{O}}\Bigg(  B\epsilon\log{T} + (\sqrt{2}+2) \sqrt{SANT} \sqrt{\log{\frac{SANT}{\epsilon}}} \Bigg),\nonumber
\end{align*}
and with probability at least $ (1-(\frac{\epsilon}{{SANT}})^{\frac{1}{2}})^2$, \fairucrl achieves the fairness violation regret  for each arm $\forall n\in\cN$  as:
\begin{align*}
\Delta_T^{n,F} = \tilde{\mathcal{O}} \Bigg(  \eta_n  \epsilon \log{T}   +   C_0 T_{\text{Mix}}^n    \sqrt{SANT}\log\!{\frac{SANT}{\epsilon}}  \Bigg),
\end{align*}
where $B$ is the activation budget, $\epsilon$ is the constant defined to build confidence interval, 
$T_{\text{mix}}^n$ is the mixing time of the true MDP associated with arm $n$, $C_0 = 4 (\sqrt{2} +1)\big( \hat{n} + \frac{C\rho^{\hat{n}}}{1-\rho} \big)$ and $\hat{n} = \lceil \log_\rho{C^{-1}} \rceil$ with $\rho$ and $C$ being constants (see Corollary \ref{coro:sensitivity} in supplementary materials).
\end{theorem} 
As discussed in Remark~\ref{remark:alg1}, the design of \fairucrl differs from UCRL type algorithms in several aspects.  These differences further necessitate different proof techniques for regret analysis. 
First, we leverage the relative value function of Bellman equation for long-term average MDPs, which enables us to transfer the regret to the difference of relative value functions. Thus, only the first moment behavior of the transition kernels are needed to track the regret, while state-of-the-art \cite{wang2020restless} leveraged the higher order moment behavior of transition kernels for a specific MDP, which is hard for general MDPs.  Closest to ours is \citet{xiong2022learning}, which however bounded the reward regret under the assumption that the diameter $D$ of the underlying MDP associated with each arm is known.  
Unfortunately, this knowledge is often unavailable in practice and there is no easy way to characterize the dependence of $D$ on the number of arms $N$ \cite{akbarzadeh2022learning}. Finally, in conventional regret analysis of RL algorithms for \rmab, e.g., \citet{akbarzadeh2022learning,xiong2022reinforcement,xiong2022learning,wang2020restless}, only the reward regret is bounded. However, for our \frmab with long-term fairness among each arm, we also need to characterize the fairness violation regret, for which, we leverage the mixing time of the underlying MDP associated with each arm. This is one of our main theoretical contributions that differentiates our work. 

We note that another line of works on constrained MDPs (CMDPs) either considered a similar extended LP approach  ~\cite{kalagarla2020sample,efroni2020exploration}  in a finite-horizon setting, which differ from our infinite-horizon setting, or are only with a long-term cost constraint \cite{singh2020learning,chen2022learning}, while our \rmab problem not only has a long-term fairness constraint, but also an instantaneous activation constraint that must be satisfied at each decision epoch. This makes their approach not directly applicable to ours.

\subsection{Proof Sketch of Theorem~\ref{thm:regret}}
We present some lemmas that are essential to prove Theorem \ref{thm:regret}. 
Our proof consists of three steps: regret decomposition and regret characterization when the true MDPs are in the confidence ball or not. 
A key challenge lies in bounding the fairness violation regret, for which the decision variable is the action $a$ in our \fairucrl, while most recent works, e.g., \citet{efroni2020exploration, xiong2022learning, akbarzadeh2022learning} focused on the reward function of the proposed policy. This challenge differentiates the proof, especially on bounding the fairness violation regret when the true MDP belongs to the confidence ball. To start with, we first introduce a lemma for the decomposition of reward and fairness violation regrets:

\begin{lemma}\label{lem:11}
The reward and fairness violation regrets of \fairucrl can be decomposed into the summation of $k$ episodic regrets with a constant term with probability at least $1-(\frac{\epsilon}{{SANT}})^{\frac{1}{2}}$., i.e.
\begin{align*}
&\Delta^R_T\{\pi^{*,k}, \forall k\}\leq \sum\limits_{k=1}^{K} \Delta^R_k\{\pi^{*,k}\}+ \sqrt{\frac{1}{4}T\log\frac{SANT}{\epsilon}}, \displaybreak[0]\\
&\Delta^{n,F}_T\{\pi^{*,k}, \forall k\}\leq \sum\limits_{k=1}^{K} \Delta^{n,F}_k\{\pi^{*,k}\}+ \sqrt{\frac{1}{4}T\log\frac{SANT}{\epsilon}}, 
\end{align*}
 where $\Delta_k^R$ and $\Delta_k^{n, F}$ are the reward/ fairness violation regret in episode $k$ under policy $\pi^{*,k}$.
\end{lemma}
\emph{Proof Sketch:} With probability of at least $1-(\frac{\epsilon}{{SANT}})^{\frac{1}{2}}$, the difference between reward until time $T$ and the episodic reward for all $K$ episodes can be bounded with a constant term $\sqrt{\frac{1}{4}T\log\frac{SANT}{\epsilon}} $  via Chernoff-Hoeffding’s inequality. This is in parallel with several previous works, e.g. \citet{akbarzadeh2022learning, xiong2022learning, efroni2020exploration}.

\noindent \textbf{Proof Sketch of Fairness Violation Regret.} The proof of fairness violation regret is one of our main theoretical contributions in this paper. To our best knowledge, this is the first result for online \frmab, i.e. with both instantaneous activation constraint and long-term fairness constraint. We now present two key lemmas which are essential to bound the fairness violation regret when combining with Lemma~\ref{lem:11}. 

First, we show that the fairness violation regret can be bounded when the transition and reward function of true MDP (denoted by $M$) does not belong to the confidence ball, i.e. $M\notin  \mathcal{M}_k$. 
\begin{lemma} \label{lemma:fairfail}
The fairness violation regret for failing confidence ball for all $K$ episodes is bounded by $$\sum\limits_{k=1}^{K} \Delta_k^{n,F} \{\pi^{*,k},\forall k\}\mathds{1} (M\notin \mathcal{M}_k)\leq \frac{1}{2} \eta_n  \epsilon \log{T} .$$
\end{lemma}
\emph{Proof Sketch:} With the probability of failing event $P(M\notin \mathcal{M}_k) \leq \frac{\epsilon}{kH}$, one can bound the fairness violation term since $\eta_n - a_n(t)\leq \eta_n$. The final bound is obtained by summing over all episodes.

Now, we present the dominated term in bounding the fairness violation regret when the true MDP belongs to the confidence ball.
\begin{lemma}\label{lemma:goodeventfair}
    The fairness violation regret when the true MDP belongs to the confidence ball in each episode $k$ is bounded by 
    \begin{align*}
    &\sum_{k=1}^K \Delta^{n,F}_k\{\pi^{*,k},\forall k\} \mathds{1} (M\notin \mathcal{M}_k) \\
    \leq&  C_0 T_{\text{Mix}}^n \bigg((\sqrt{2}+1)\sqrt{SANT}\sqrt{\log{\frac{SANT}{\epsilon}}}  \\
    &\qquad \qquad \qquad+\frac{1}{2} \sqrt{T} \log{\frac{SANT}{\epsilon}} \bigg) + \sqrt{T}\frac{C}{1-\rho}.
    \end{align*}
\end{lemma}
\emph{Proof Sketch:}  We first define a new variable $\overline{F}_n(\pi^k,p) = \frac{1}{T} \lim_{T\rightarrow\infty} (\sum_{t=1}^T a_n(t)|\pi^k,p)$ as the long term average fairness variable under policy $\pi^k$ for arm $n$ with MDP that has the true transition probability matrix $p$. We show that the fairness violation regret when the true MDP belongs to confidence ball can be upper bounded by $\mathbb{E} \big[ H\eta_n - \overline{F}_n(\pi_n^k,p) \big] $ with a constant term.

Next we introduce another variable close to $\overline{F}_n(\pi^k,p)$, that is $\overline{F}_n(\pi^k,\theta) = \frac{1}{T} \lim_{T\rightarrow\infty} (\sum_{t=1}^T a_n(t)|\pi^k,\theta)$ as the fairness variable under policy $\pi$ in episode $k$ for arm $n$ with  MDP whose transition matrix $\theta$ belongs to the confidence ball. By comparing the total variance norm of $\overline{F}_n(\pi^k,p)$ and $\overline{F}_n(\pi^k,\theta)$,  we can upper bound $\mathbb{E} \big[ H\eta_n - \overline{F}_n(\pi_n^k,p) \big] $ as $\beta_n^k(\pi_k) := 2 \big( \hat{n} + \frac{C\rho^{\hat{n}}}{1-\rho} \big)\max_{s} \sum_a \pi_k(s,a) \delta_n^k(s,a)$, where $\pi_k$ is the policy in episode $k$. In order to bound $\beta_n^k(\pi_k)$ with the expected number of counts of $(s,a)$  pair in episode $k$ $\mathbb{E} [c_n^k(s,a)]$, we leverage the mixing time $T_{mix}^n$.

The regret is further split into two terms, one of which $\sum_{k=1}^K \sum_{(s,a)} \sum_n \frac{c_n^k(s,a)}{C_n^{k-1}(s,a)}$ can be bounded as $(\sqrt{2}+1)  \sqrt{SANT} $ through the induction of sequence summation, while the other term $\sum_{k=1}^K  \sum_{(s,a)} \frac{\mathbb{E} [c_n^k(s,a)] - c_n^k(s,a)}{\sqrt{2C_{n}^{k-1}(s,a)}}$ can be upper bounded by $\sqrt{T} \sqrt{\frac{1}{4}\log{\frac{SANT}{\epsilon}}}$ via Azuma-Hoeffding’s inequality, as it can be considered as a martingale difference sequence.

\noindent \textbf{Proof Sketch of Reward Regret.}
Similar to the fairness violation regret, we first bound the reward regret  when the MDP does not belong to the confidence ball.
\begin{lemma}\label{lemma:failregret}
    The reward regret for failing the confidence ball for all $K$ episodes is bounded by $$\sum_{k=1}^K \Delta^R_k\{\pi^{*,k},\forall k\} \mathds{1}(M\notin \mathcal{M}_k) \leq B\epsilon\log{T}. $$
\end{lemma}
\emph{Proof Sketch:} Similar to Lemma \ref{lemma:fairfail}, the probability of failing confidence ball is bounded by $P(M\notin \mathcal{M}_k) \leq \frac{\epsilon}{kH}.$ Summing over all episodes yields the bound. 

We then present the dominated term in the reward regret. 

\begin{lemma}\label{lemma:goodevent}
    The reward regret when the true MDP belongs to the confidence ball in each episode $k$ is bounded by 
    \begin{align*}
    \sum_{k=1}^K \Delta^R_k\{\pi^{*,k},\forall k\} \mathds{1}& (M\in \mathcal{M}_k) \nonumber\\ 
    &\leq  (\sqrt{2}+2) \sqrt{SANT}\sqrt{\log{\frac{SANT}{\epsilon}}}.
    \end{align*}
\end{lemma}
\emph{Proof Sketch:} We split the reward regret into two terms, $\sum_{(s,a)} \sum_n  c_{n}^k(s,a)(\mu^* /B - \tilde{r}_n(s,a))$ and $\sum_{(s,a)} \sum_n c_{n}^k(s,a) 2 \sqrt{\frac{1}{2C_{n}^{k-1}(s,a)}\log{\frac{SANkH}{\epsilon}}}$. The first term is upper bounded by 0 due to the fact that for any episode $k$, the optimistic average reward $\tilde{r}_n (s,a)$ of the optimistically chosen policy ${\tilde{\pi}}_k$ within the confidence ball is equal or larger than the true optimal average reward $\mu^*$, provided that the true MDP belongs to confidence ball. Similar to Lemma \ref{lemma:goodeventfair}, the second term can be bounded with $(\sqrt{2}+1)  \sqrt{SANT}$.

\begin{figure*}[t]
 \centering
\begin{minipage}{.24\textwidth}
 \centering
 \includegraphics[width=1\columnwidth]{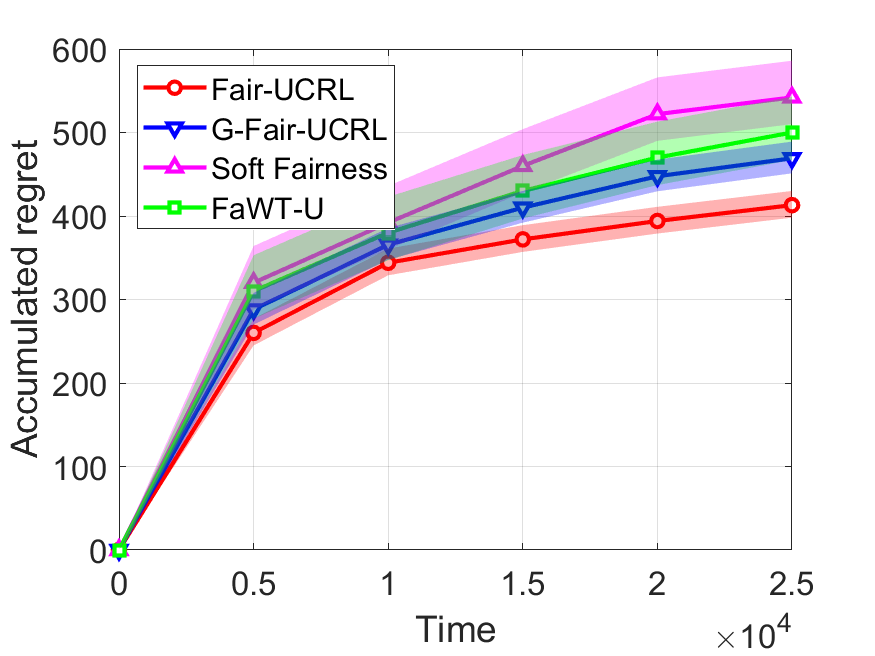}

\caption{Reward regret in simulated environments.}
 \label{fig:regetsyn}
 \end{minipage}\hfill
   \begin{minipage}{.24\textwidth}
 \centering
 \includegraphics[width=1\columnwidth]{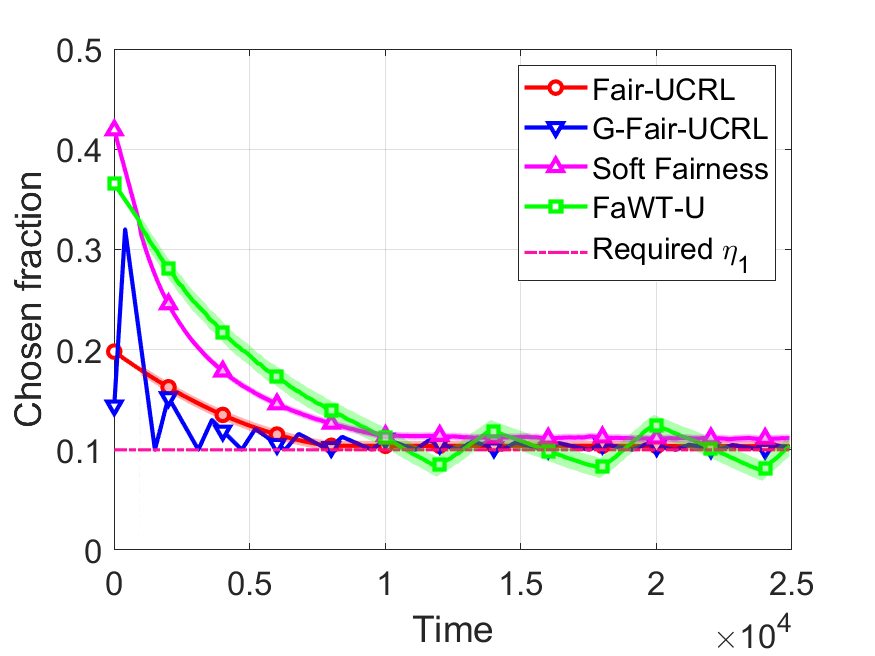}

\caption{Fairness constraint violation for arm 1.}
\label{fig:syncon1}
 \end{minipage}\hfill
 \begin{minipage}{.24\textwidth}
 \centering
 \includegraphics[width=1\columnwidth]{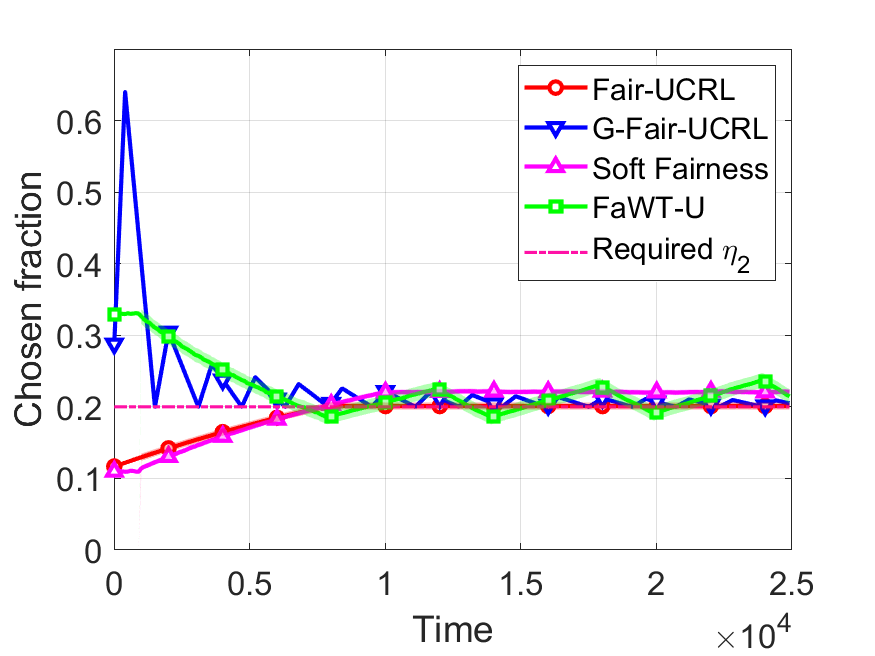}

 \caption{Fairness constraint violation for arm 2.}
\label{fig:syncon2}
 \end{minipage}\hfill
   \begin{minipage}{.24\textwidth}
 \centering
 \includegraphics[width=1\columnwidth]{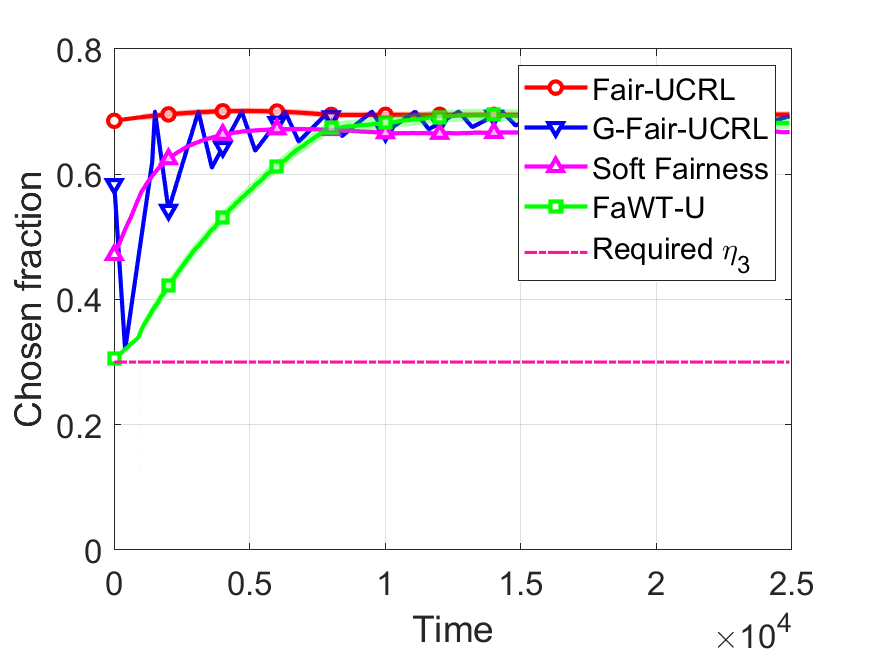}
 
\caption{Fairness constraint violation for arm 3.}
\label{fig:syncon3}
 \end{minipage}

 \end{figure*}

\begin{figure*}[t]
 \centering
\begin{minipage}{.33\textwidth}
 \centering
 \includegraphics[width=1\columnwidth]{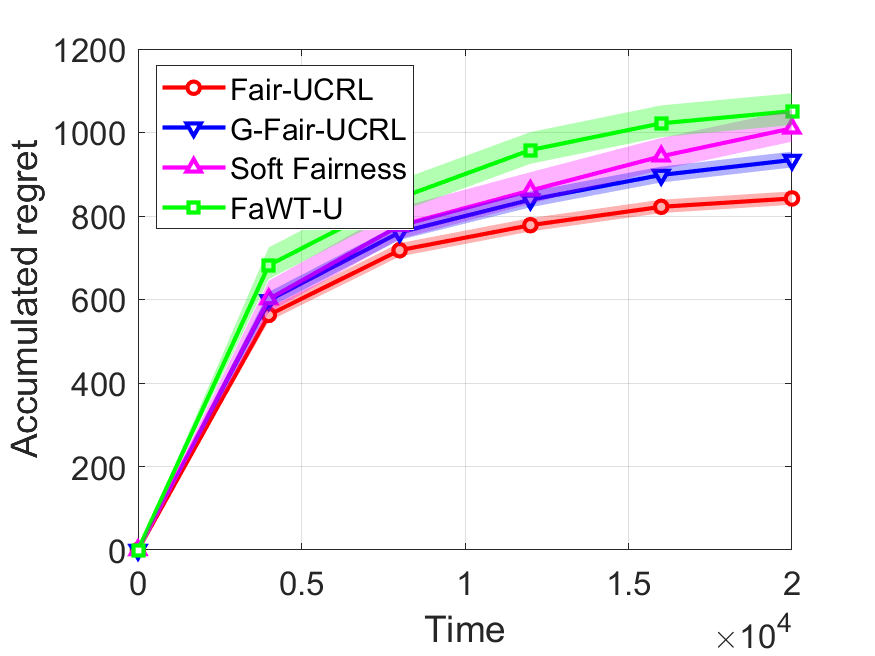}
 
\subcaption{Reward regret.}
 \label{fig:cpapregret}
 \end{minipage}\hfill
   \begin{minipage}{.33\textwidth}
 \centering
 \includegraphics[width=1\columnwidth]{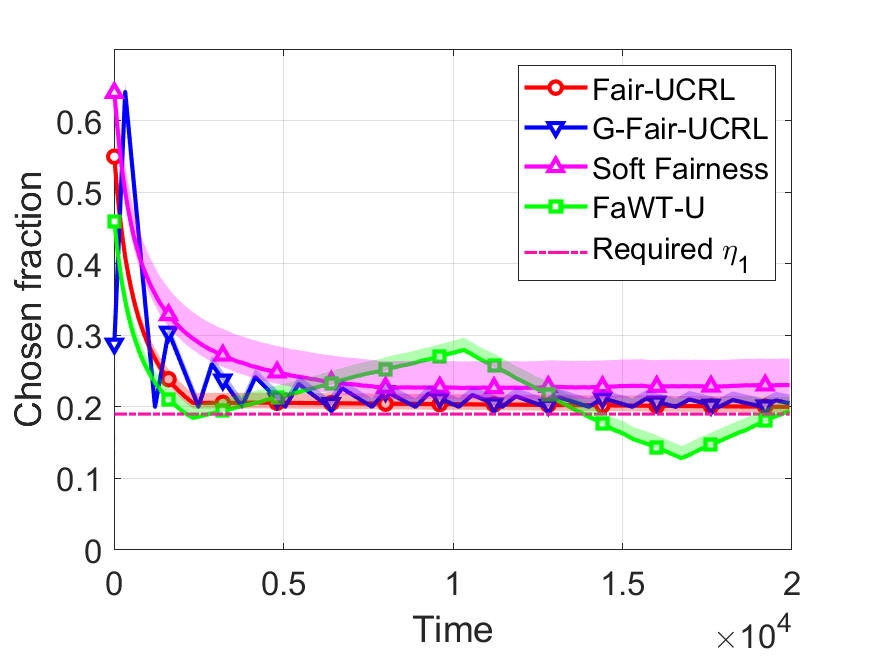}

\subcaption{Fairness constraint violation for arm 1.}
\label{fig:cpapcon1}
 \end{minipage}\hfill
 \begin{minipage}{.33\textwidth}
 \centering
 \includegraphics[width=1\columnwidth]{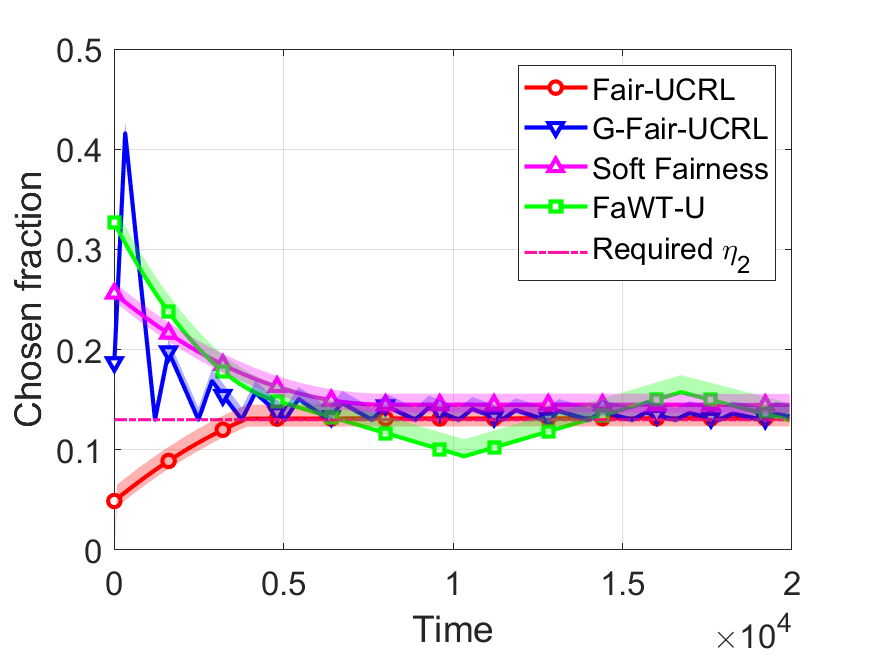}

 \subcaption{Fairness constraint violation for arm 2.}
\label{fig:cpapcon2}
 \end{minipage}\hfill

 \caption{Continuous positive airway pressure therapy.} 
 \end{figure*}


\section{Experiments}\label{sec:exp}

In this section, we first evaluate the performance of \fairucrl in simulated environments, and then demonstrate the utility of \fairucrl by evaluating it under three real-world applications of \rmab. 

\subsection{Evaluation in Simulated Environments}

\textbf{Settings.} We consider $3$  classes of arms, each including 100 duplicates with state space  $\mathcal{S} \in\{0,1,2,3,4,5\}$.  Class-$n$ arm arrives with rate $\lambda_n = 3n$ for $n=1,2,3$, and departs with a fixed rate of $\mu = 5$. We consider a controlled Markov chain in which states evolve as a specific birth-and-death process, i.e., state $s$ only transits to $s+1$ or $s-1$ with probability $P(s,s+1) = {\lambda}/{(\lambda+\mu)}$ or $P(s,s-1) = {\mu}/{(\lambda+\mu)}$, respectively.  Class-$n$ arm generates a random reward $r_n(s) \sim Ber(sp_n)$, with $p_n$ uniformly sampled from $[0.01,0.1]$. The activation budget is set to 100.  The minimum activation fraction $\eta$ is set to be {0.1, 0.2 and 0.3} for the three classes of arms, respectively. We set $K=H=160$. We use Monte Carlo simulations with $1,000$ independent trials.

\noindent\textbf{Baselines.} We compare \fairucrl with three baselines: (1) \textit{FaWT-U} \citep{li2022efficient} activates arms based on their Whittle indices.  If the fairness constraint is not met for an arm after a certain time, FaWT-U always activates that arm regardless of its Whittle index.  (2) \textit{Soft Fairness} \citep{li2022towards} incorporates softmax based value iteration method into the \rmab setting. 
Since both algorithms are designed for infinite-horizon discounted reward settings, we choose the discounted factor to be 0.999 for fair comparisons with our \fairucrl, which is designed for infinite-horizon average-reward settings. 
(3) \gfairucrl: We modify our proposed \fairucrl by greedily enforcing the fairness constraint satisfaction in each episode. Specifically, at the beginning of each episode, \gfairucrl randomly pulls an arm to force each arm $n$ to be pulled $H\eta_n$ times. This greedy exploration will take $\lceil \frac{\sum_{n=1}^N H\eta_n}{B}\rceil$ decision epochs in total in each episode. \gfairucrl then operates in the same manner as \fairucrl in the rest of this episode. More details on \gfairucrl are provided in supplementary materials.

\noindent\textbf{Reward Regret.} The accumulated reward regrets are presented in Figure~\ref{fig:regetsyn}, where we use Monte Carlo simulations with $1,000$ independent trials. \fairucrl achieves the lowest accumulated reward regret. More importantly, this is consistent with our theoretical analysis (see Theorem~\ref{thm:regret}), while neither FaWT-U nor Soft Fairness provides a finite-time analysis, i.e., nor provable regret bound guarantees.

\noindent\textbf{Fairness Constraint Violation.} The activation fraction for each arm over time under different policies are presented in Figures~\ref{fig:syncon1},~\ref{fig:syncon2} and~\ref{fig:syncon3}, respectively.  After a certain amount of time, the minimum activation fraction for each arm under \fairucrl is always satisfied, and a randomized initialization may cause short term fairness violation, for example, after $6,500$ time steps for arm 2, even though the constraint needs to be satisfied on average. 
Similar observations hold for Soft Fairness, while for FaWT-U, fairness constraint violation repeatedly occurs over time for arm 1 and arm 2.

\begin{figure*}[t]
\centering
\begin{minipage}{.33\textwidth}
 \centering
 \includegraphics[width=1\columnwidth]{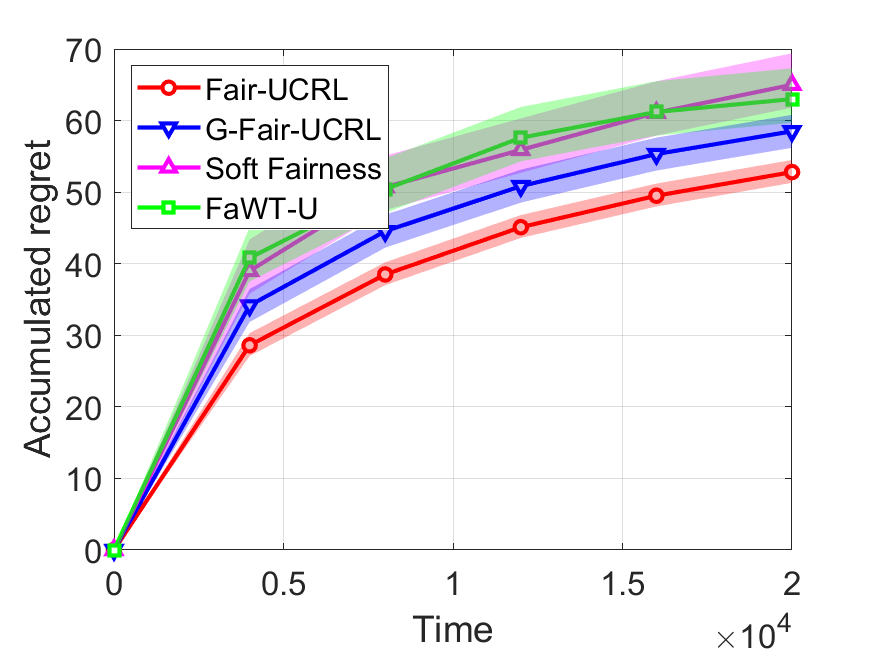}
 
\subcaption{Reward regret.}
\label{fig:schedulingreg}
 \end{minipage}\hfill
   \begin{minipage}{.33\textwidth}
 \centering
 \includegraphics[width=1\columnwidth]{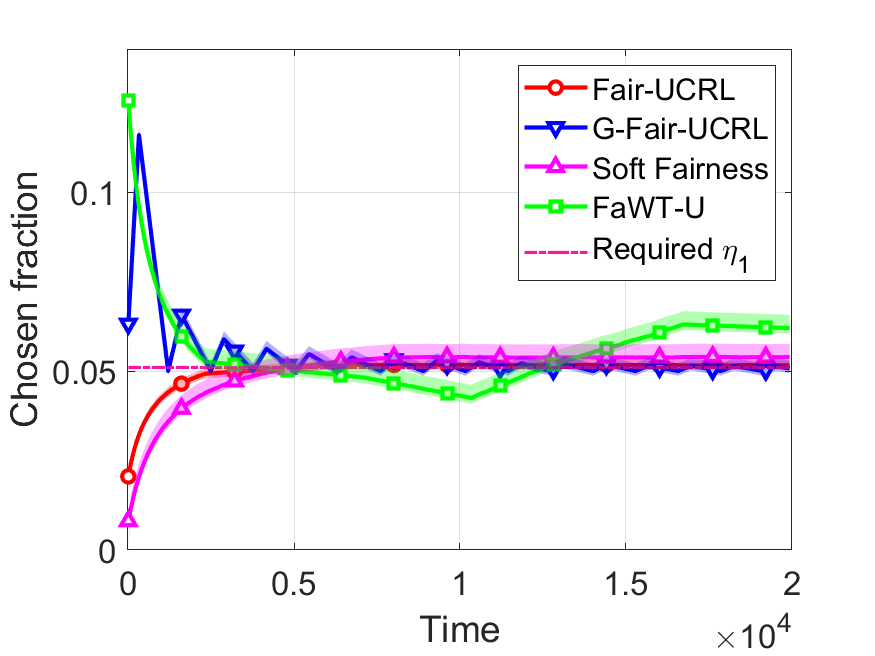}
 
\subcaption{Fairness constraint violation for arm 1.}
\label{fig:schedulingcon1}
 \end{minipage}
 \begin{minipage}{.33\textwidth}
 \centering
 \includegraphics[width=1\columnwidth]{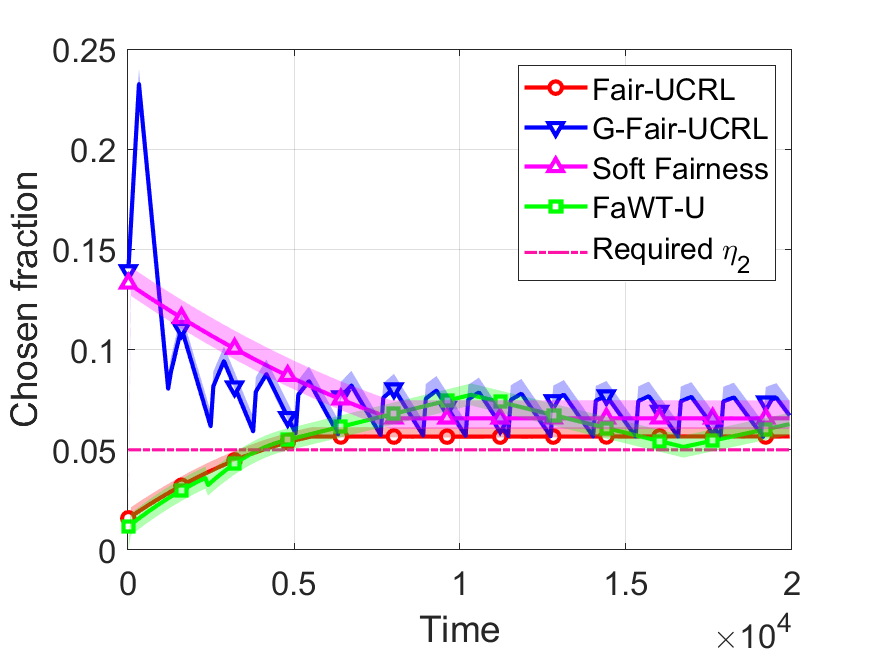}
 
\subcaption{Fairness constraint violation for arm 2.}
\label{fig:schedulingcon2}
 \end{minipage}

 \caption{PASCAL recognizing textual entailment task.}
 
 \end{figure*}

\begin{figure*}[t]
\centering
\begin{minipage}{.33\textwidth}
 \centering
 \includegraphics[width=1\columnwidth]{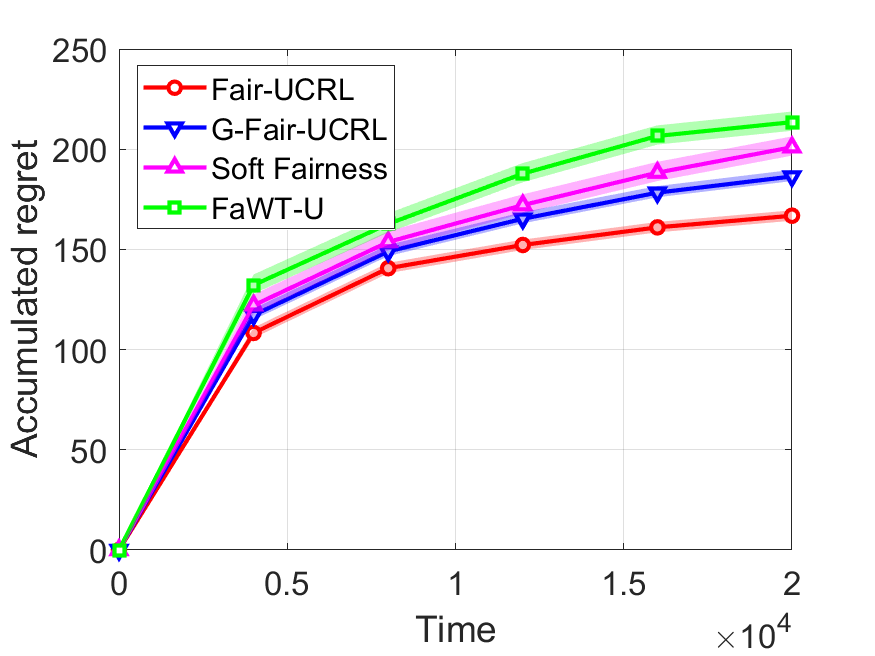}
 
\subcaption{Reward regret.}
\label{fig:landreg}
 \end{minipage}\hfill
   \begin{minipage}{.33\textwidth}
 \centering
 \includegraphics[width=1\columnwidth]{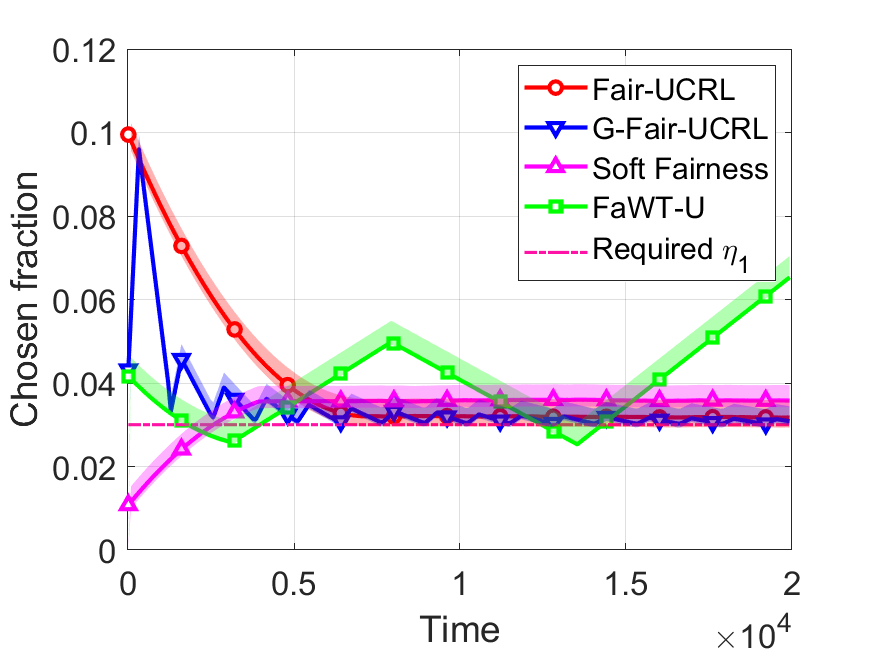}

\subcaption{Fairness constraint violation for arm 1.}
\label{fig:landcon}
 \end{minipage}
 \begin{minipage}{.33\textwidth}
 \centering
 \includegraphics[width=1\columnwidth]{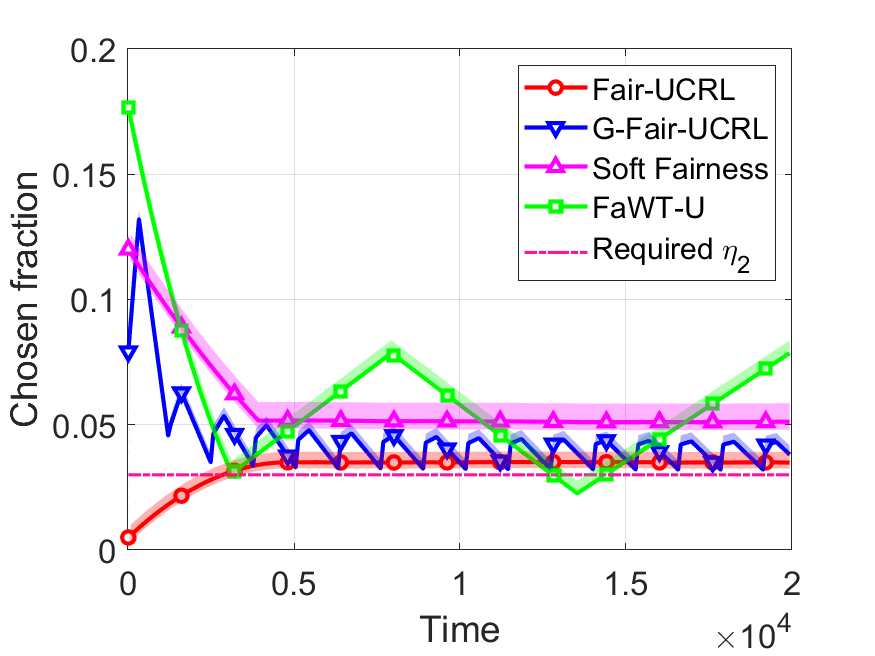}

\subcaption{Fairness constraint violation for arm 2.}
\label{fig:landcon2}
 \end{minipage}
 
 \caption{Land mobile satellite system.}

 \end{figure*}

\subsection{Continuous Positive Airway Pressure Therapy} 
We study the continuous positive airway pressure therapy (CPAP) as in \citet{herlihy2023planning,li2022towards}, which is a highly effective treatment when it is used consistently during the sleeping for adults with obstructive sleep apnea.  Similar non-adherence to CPAP in patients hinders the effectiveness, 
we adapt the Markov model of CPAP adherence behavior \citep{kang2013markov} to a two-state system with the clinical adherence criteria.  Specifically, there are 3 states, representing low, intermediate and acceptable adherence levels. Patients are clustered into two groups, ``Adherence'' and ``Non-Adherence''. 
The first group has a higher probability of staying in a good adherence level.  There are 20 arms/patients with 10 in each group.  
The transition matrix of arms in each group contains a randomized, small noise from the original data. The intervention, which is the action applied to each arm, results in a 5\% to 50\% increase in adherence level. The budget is $B = 5$ and  the fairness constraint is set to be a random number between [0.1, 0.7].  The objective is to maximize the total adherence level. The accumulated reward regret and the activation fraction for two randomly selected arms are presented in Figures~\ref{fig:cpapregret},~\ref{fig:cpapcon1} and~\ref{fig:cpapcon2}, respectively.  Again, we observe that \fairucrl achieves a much smaller reward regret and the fairness constraint is always satisfied after a certain amount of time.

\subsection{PASCAL Recognizing Textual Entailment}
We study the PASCAL recognizing textual entailment task
as in \citet{snow2008cheap}.  Workers are assigned with tasks that determine if \emph{hypothesis} can be inferred from \emph{text}. There are 10 workers. Due to lack of background information, a worker may not be able to correctly annotate a task. We assign a ``successful annotation probability'' to each worker, which is based on the average success rate over 800 tasks in the dataset. Each worker is a MDP with state 1 (correctly annotated) and 0 (otherwise).  The transition probability from state 0 to 1 with $a=1$ is the same as that of staying at state 1 with $a=1$, which is set as the successful annotation probability.  Reward is $1$ if a selected worker successfully annotates the task, and $0$ otherwise.  At each time, 3 tasks are generated (i.e., $B = 3$) and distributed to workers. 
Fairness constraints for all workers are set to be $\eta = 0.05$.  Again, both proposed algorithms outperform two baselines and maintain higher selection fraction as shown in Figures~\ref{fig:schedulingreg}, ~\ref{fig:schedulingcon1} and~\ref{fig:schedulingcon2} for two randomly selected arms, respectively.

\subsection{Land Mobile Satellite System}
We study the land mobile satellite system problem as in \citet{https://doi.org/10.1002/sat.964}, in which the land mobile satellite broadcasts a signal carrying multimedia services to handheld devices. There are 4 arms with different elevation angles ($40^{\circ}, 60^{\circ}, 70^{\circ}, 80^{\circ}$) of the antenna in urban area. Only two states (\emph{Good} and \emph{bad}) are considered and we leverage the same transition matrix as in \citet{https://doi.org/10.1002/sat.964}. Similar, we use the average direct signal mean as the reward function. The budget is  $B=2$. We apply the fairness constraint $\eta = 0.03$ to all angles. Again, \fairucrl outperforms the considered baselines in reward regret (Figure~\ref{fig:landreg}), while satisfies long term average fairness constraint (Figures~\ref{fig:landcon} and~\ref{fig:landcon2} for two randomly selected arms).

\clearpage

\section*{Acknowledgements} 

This work was supported in part by the National Science Foundation (NSF) grants 2148309 and 2315614, and was supported in part by funds from OUSD R\&E, NIST, and industry partners as specified in the Resilient \& Intelligent NextG Systems (RINGS) program. This work was also supported in part by the U.S. Army Research Office (ARO) grant W911NF-23-1-0072, and the U.S. Department of Energy (DOE) grant DE-EE0009341. Any opinions, findings, and conclusions or recommendations expressed in this material are those of the authors and do not necessarily reflect the views of the funding agencies.
\bibliography{aaai24}
\clearpage
\appendix

\section{Related Work}

\textbf{Offline \rmab.} The \rmab was first introduced in \cite{whittle1988restless}, which is known to be computationally intractable and is in general PSPACE hard \cite{papadimitriou1994complexity}.  The state-of-the-art approach to \rmab is the Whittle index policy, which is provably asymptotically optimal \cite{weber1990index}. However, Whittle index is well-defined only when the {indexability condition} is satisfied, which is in general hard to verify.  Furthermore, even when an arm is indexable, finding its Whittle index can still be intractable \cite{nino2007dynamic}.  As a result, Whittle indices of many practical problems remain unknown.  
Exacerbating these limitations is the fact that these Whittle-like index policies fail to provide any guarantee on how activation is distributed among arms.  To address  fairness concerns, there are recent efforts on imposing fairness constraints in \rmab, e.g., \citet{li2022efficient,herlihy2023planning}.  However, they either considered the finite-horizon or infinite-horizon discounted reward setting, and hence cannot be directly applied to the more challenging infinite-horizon average-reward setting with long-term fairness constraints considered in this paper.  In addition, there is no rigorous finite-time performance analysis in terms of regret in \citet{li2022efficient,herlihy2023planning}. The probabilistic fairness guarantees in \citet{herlihy2023planning} are restricted to a 2-state MDP and cannot be easily extended to a general MDP as considered in this paper.

\noindent\textbf{Online \rmab.}
Since the underlying MDPs associated with each arm in \rmab are unknown in practice, it is important to examine \rmab from a learning perspective, e.g., \cite{dai2011non,tekin2011adaptive,liu2011logarithmic,liu2012learning,tekin2012online,ortner2012regret,jung2019regret,jung2019thompson,wang2020restless}. However, these methods either leveraged a heuristic policy and  may not perform close to the offline optimum, or are computationally expensive.   
Recently,  low-complexity RL algorithms have been developed for \rmab, e.g., \citet{fu2019towards,avrachenkov2022whittle,biswas2021learn,killian2021q,wang2020restless,xiong2022reinforcement,xiong2022learning,xiong2022reinforcementcache,nakhleh2021neurwin,nakhleh2022deeptop}.  
However, none of these studies are applicable to  \frmab, where the DM not only needs to maximize the long-term reward under the instantaneous activation constraint as in \rmab, but ensures that a long-term fairness constraint is satisfied for each arm.  Our RL algorithm with rigorous regret analysis on both the reward and the fairness violation further differentiate our work.  

\section{The \fairucrl Algorithm }

\textbf{Algorithm Overview.} 
\fairucrl proceeds in episodes as summarized in Algorithm~\ref{alg:UCRL}.  Let $\tau_k$ be the start time of episode $k$.  Each episode consists of two phases: (i) \textit{\textbf{Optimistic planning}}: 
 At the beginning of each episode (line 4 in Algorithm~\ref{alg:UCRL}), \fairucrl constructs a confidence ball that contains a set of plausible MDPs \citep{jaksch2010near} for each arm $\forall n\in\cN$.  To obtain an optimistic estimate of the true transition kernels and rewards, \fairucrl solves an {optimistic planning} problem with parameters chosen from the constructed confidence ball.  Since the online \frmab is computationally intractable, we first relax the instantaneous activation constraint to a ``long-term activation constraint'', from which we obtain a linear programming (LP) where decision variables are the occupancy measures \cite{altman1999constrained} corresponding to the process associated with $N$ arms.  We refer to the planning problem as {an extended LP} (line 5 in Algorithm~\ref{alg:UCRL}), which is described in details below.  (ii) \textit{\textbf{Policy execution}}: Unfortunately, the solutions to this extended LP is not always feasible to the online \frmab, which requires the instantaneous activation constraint to be satisfied at each decision epoch, rather than in the average sense as in the extended LP.  To address this challenge, \fairucrl constructs a so-called \fair index policy on top of the solutions to the extended LP, which is executed during the policy execution phase of each episode (lines 7-8 in Algorithm~\ref{alg:UCRL}).

\subsubsection{Optimistic Planning.}\label{sec:planning}
\fairucrl maintains two counts for each arm $n$. Let $C_{n}^{k-1}(s,a)$ be the number of visits to state-action pairs $(s,a)$ until $\tau_k$, and $C_{n}^{k-1}(s,a, s^\prime)$ be the number of transitions from $s$ to $s^\prime$ under action $a$ until $\tau_k$. 
At the end of episode $k$, \fairucrl updates these counts as $C_{n}^{k}(s,a)=C_{n}^{k-1}(s,a)+\sum_{h=1}^{H}\mathds{1}(s_{n}^{k}(h)=s, a_{n}^{k}(h)=a),$ and $C_{n}^{k}(s,a, s^\prime)=C_{n}^{k-1}(s,a, s^\prime)+\sum_{h=1}^{H}\mathds{1}(s_{n}^{k}(h+1)=s^\prime|s_{n}^{k}(h)=s, a_{n}^{k}(h)=a)$, $\forall (s,a)\in\cS\times\cA$ and $\forall (s,a,s^\prime)\in\cS\times\cA\times\cS$ for each arm $n$, where $s_{n}^{k}(h)$ is the state of arm $n$ at the $h$-th time frame in episode $k$.  
Then \fairucrl estimates the true transition kernel and the true reward function by the corresponding empirical averages as:
\begin{align}\label{eq:empirical_est_P}
&\hat{P}_{n}^{k}(s^\prime|s,a)\!=\!\frac{C_{n}^{k-1}(s,a,s^\prime)}{\max\{C_{n}^{k-1}(s,a),1\}},\displaybreak[0]\\
&\hat{r}_n^{k}(s,a)\!=\!
\frac{\sum\limits_{l=1}^{k-1}\sum\limits_{h=1}^{H}r_n^l(s,a)\mathds{1}(s_n^{l}(h)= s, a_n^{l}(h)=a)}
{\max\{C_n^{k-1}(s,a),1\}}. 
\end{align}

\fairucrl further defines  confidence intervals for transition probabilities and rewards so that true transition probabilities and  true rewards lie in them with high probabilities, respectively.  Formally, for $\forall (s,a)\in\cS\times\cA$, we define 
\begin{align}
 \mathcal{P}_{n}^{k}(s, a) &:=\{\tilde{P}_{n}^k(s^\prime|s, a), \forall s^\prime: \nonumber\displaybreak[0]\\
  \qquad\qquad &|\tilde{P}_{n}^k(s^\prime|s, a) - \hat{P}_{n}^{k}(s^\prime|s, a)| \leq  \delta_{n}^{k}(s, a)\}, \displaybreak[1]\label{eq:confidence_ball} \\
 \mathcal{R}_{n}^{k}(s, a)& :=\{\tilde{r}_{n}^k(s, a):\nonumber\displaybreak[2] \\
 \qquad\qquad &|\tilde{r}_{n}^k(s, a) - \hat{r}_{n}^{k}(s, a)| \leq  \delta_{n}^{k}(s, a)\}\label{eq:confidence_ballr} ,
\end{align}
where the size of the confidence intervals $\delta_{n}^{k}(s,a)$ is built according to the Hoeffding inequality \citep{maurer2009empirical} for $\epsilon\in(0,1)$ as
\begin{align}\label{eq:confidenceradius}
\delta_{n}^{k}(s,a)=\sqrt{\frac{1}{2C_{n}^{k-1}(s,a)}\log\bigg(\frac{SAN(k-1)H}{\epsilon}\bigg)}.
\end{align}

To this end, we define a set of plausible MDPs associated with the confidence intervals in episode $k$: 
\begin{align}\label{eq:plausible2}
      \mathcal{M}^k &=  \big\{M_n^k  = (\cS, \cA, \tilde{P}_n^k,  \tilde{r}_n^k): 
      |\tilde{P}_n^k(s^\prime|s, a) - \hat{P}_n^{k}(s^\prime|s, a)| \nonumber\\
      &\leq  \delta_n^{k}(s, a), 
      \tilde{r}_n^k(s, a) = \hat{r}_n^{k}(s, a) + \delta_n^{k}(s, a)\big\}.  
\end{align}

\fairucrl computes a policy $\pi^{k,*}$ by performing optimistic planning.  In other words, in each episode $k,$ given the set of plausible MDPs~(\ref{eq:plausible2}), \fairucrl selects an optimistic MDP $M_n^k, \forall n$ and an optimistic policy with respect to \frmab($\tilde{P}_n^k, \tilde{r}_n^k, \forall n$), which is similar to the offline \frmab($P_n, r_n, \forall n$) in~(\ref{eq:FRMAB-obj})-(\ref{eq:FRMAB-constraint2}) by replacing transition and reward functions with $\tilde{P}_n^k(\cdot|\cdot, \cdot)$ and $\tilde{r}_n^k(\cdot, \cdot)$, $\forall s,a,s^\prime,n,k$, respectively, in the confidence intervals~(\ref{eq:confidence_ball}),(\ref{eq:confidence_ballr}) due to the fact that the corresponding true values are not available.

As aforementioned, it is well known that solving \frmab($\tilde{P}_n^k, \tilde{r}_n^k, \forall n$) is intractable even in the offline setting \citep{whittle1988restless}.  
To address this challenge, we first relax the instantaneous activation constraint so as to achieve a ``long-term activation constraint'', i.e., the activation.  For simplicity, we call \frmab($\tilde{P}_n^k, \tilde{r}_n^k, \forall n$) with the long-term activation constraint as ``the relaxed \frmab($\tilde{P}_n^k, \tilde{r}_n^k, \forall n$)''. It turns out that this relaxed \frmab($\tilde{P}_n^k, \tilde{r}_n^k, \forall n$) can be equivalently transformed into a linear programming (LP) via replacing all random variables in the relaxed \frmab($\tilde{P}_n^k, \tilde{r}_n^k, \forall n$) with the occupancy measure corresponding to each arm $n$ \citep{altman1999constrained}.

\textbf{Extended LP.}
We cannot solve this LP since we have no knowledge about true transition kernels and rewards.  Thus we further rewrite it as an extended LP by leveraging \textit{state-action-state occupancy measure} $z_{n}^k(s, a, s^\prime)$ 
to express confidence intervals of transition probabilities: given a policy $\pi$ and transition functions $\tilde{P}_n^k,$ the occupancy measure $z_{n}^k(s, a, s^\prime)$ induced by $\pi$ and $\tilde{P}_n^k$ is that $\forall n, s,s^\prime, a, k$:
    $z_n^k(s,a, s^\prime)\!:=\! \lim_{H\rightarrow \infty}\frac{1}{H}\mathbb{E}_\pi [ \sum_{h=1}^{H-1} \mathds{1} (s_n(h)\!=\! s, a_n(h)\!=\!a, s_n(h+1)= s^\prime)].$

The extended LP over $z^k:=\{z_{n}^k(s, a, s^\prime), \forall n\in\cN\}$ is then given as $\textbf{ELP}(\mathcal{M}^k, z^k):$
\begin{align}
\min_{z^k}&~\sum_{n=1}^{N}\sum_{(s,a,s^\prime)} z_{n}^k(s,a, s^\prime)\tilde{r}_n^k(s,a) \label{eq:UCB_extended-obj}\\
\mbox{ s.t. } &~\sum_{n=1}^{N}\sum_{s\in\cS}\sum_{s^\prime\in\cS} z_{n}^k(s,1,s^\prime) \le B , \displaybreak[0]\label{eq:UCB_extended-cons1}\\
&~\sum_{s\in\cS}\sum_{s^\prime\in\cS}  z_{n}^k(s,1,s^\prime) \geq \eta_n, ~\forall n\in \cN, \displaybreak[1]\label{eq:UCB_extended-cons2}\\
&~\sum_{a\in\cA}\sum_{s^\prime\in\cS} z_{n}^k(s,a,s^\prime)=\sum_{s^\prime\in\cS}\sum_{a^\prime\in\cA}z_{n}^k(s^\prime, a^\prime, s), \displaybreak[2] \label{eq:UCB_extended-cons3}\\ 
&~\sum_{s\in\cS}\sum_{a\in\cA}\sum_{s^\prime\in\cS} z_{n}^k(s,a, s^\prime)=1,~\forall n\in\cN, \displaybreak[3]\label{eq:UCB_extended-cons4}\\
&~\frac{z_{n}^k(s,a,s^\prime)}{\sum_y z_{n}^k(s,a,y)}\!-\!(\hat{P}_{n}^k(s^\prime|s,a)\!+\!\delta_{n}^k(s,a))\leq 0, \label{eq:UCB_extended-cons5}\\ 
&\!-\!\frac{z_{n}^k(s,a,s^\prime)}{\sum_y z_{n}^k(s,a,y)}\!+\!(\hat{P}_{n}^k(s^\prime|s,a)\!-\!\delta_{n}^k(s,a))\leq0, \label{eq:UCB_extended-cons6} 
\end{align}
where constraints~(\ref{eq:UCB_extended-cons1})-(\ref{eq:UCB_extended-cons2}) are restatements of ``long-term activation constraint'' and ``long-term fairness constraint'', respectively; constraint~(\ref{eq:UCB_extended-cons3}) represents the fluid transition of occupancy measure; and constraint~(\ref{eq:UCB_extended-cons4}) holds since the occupancy measure is a probability measure; and the last two constraints enforce that transition probabilities are inside of the confidence ball~(\ref{eq:plausible2}). Denote the optimal solution to the extended LP as $z^{k,*}=\{z_n^{k,*}, \forall n\in\cN\}$.

\subsubsection{Policy Execution.} 
One challenge for online \frmab($\tilde{P}_n^k, \tilde{r}_n^k, \forall n$) is that the instantaneous activation constraint must be satisfied at each decision epoch, rather than in the average sense as in the extended LP~(\ref{eq:UCB_extended-obj})-(\ref{eq:UCB_extended-cons6}).  As a result, the solution to the above extended LP is not always feasible to the online \frmab($\tilde{P}_n^k, \tilde{r}_n^k, \forall n$).   Inspired by \cite{xiong2022learning}, we construct an index policy, which is feasible for the online \frmab($\tilde{P}_n^k, \tilde{r}_n^k, \forall n$).  Specifically, we derive our index policy on top of the optimal solution $z^{k,*}=\{z_n^{k,*}, \forall n\in\cN\}$.  Since $\cA=\{0,1\}$, i.e, an arm can be either active or passive at time $t$, we define the index assigned to arm $n$ in state $s_n(t)=s$ at time $t$ to be as
\begin{align}\label{eq:fair-index}
\omega_{n}^{k,*}(s):=\frac{\sum_{s^\prime}z_{n}^{k,*}(s,1,s^\prime)}{\sum_{a,s^\prime}z_{n}^{k,*}(s,a,s^\prime)}, \quad\forall  n \in\cN. 
\end{align}
We call this \textit{the fair index} since $\omega_{n}^{k,*}(s)$ represents the probability of activating arm $n$ in state $s$ towards maximizing the total rewards while guarantees fairness for arm $n$ in online \frmab.   To this end, we rank all arms according to their indices in~(\ref{eq:fair-index}) in a non-increasing order, and activate the set of $B$ highest indexed arms, denoted as $\cN(t)\subset\cN$ such that $\sum_{n\in\cN(t)} a_n^*(t)\leq B$. All remaining arms are kept passive at time $t$. We denote the resultant index-based policy, which we call the \fair index policy as $\pi^{k,*}:=\{\pi_{n}^{k,*}, \forall n\in\cN\}$, and execute this policy in this episode.

\begin{remark}\label{remark:alg}
\fairucrl draws inspiration from the infinite-horizon UCRL 
\cite{jaksch2010near}, which uses the sampled trajectory of each episode to update the plausible MDPs of next episode.  However, there exist two major differences. First, \fairucrl modifies the principle of optimism in the face of uncertainty for making decisions which is utilized by UCRL based algorithms, to not only maximize the long-term rewards but also to satisfy the long-term fairness constraint in our \frmab.  This difference is further exacerbated since the objective of conventional regret analysis, e.g., colored-UCRL2 \cite{ortner2012regret} for \rmab is to bound the reward regret, while due to the long-term fairness constraint, we also need to bound the fairness violation regret for each arm for \fairucrl, which will be discussed in details in Theorem~\ref{thm:regret}.  
Second, \fairucrl deploys the proposed \fair index policy at each episode, and thus results in solving a low-complexity extended LP, which is exponentially better than that of UCRL \cite{jaksch2010near} that need to solve extended value iterations.  We note that the design of our \fair index policy is largely inspired by the LP based approach in \citet{xiong2022learning} for online \rmab.  However, \citet{xiong2022learning} only considered the instantaneous activation constraint, and hence is not able to address the new dilemma faced by our online \frmab, which also needs to ensure the long-term fairness constraints.   
\end{remark}

\section{Proof of Theorem \ref{thm:regret}}

In this section, we  present the detailed proof for Theorem \ref{thm:regret}. As mentioned, the proof shares same structure as UCRL type of proof \citep{jaksch2010near}, and is organized in the following steps: (i) We show both reward and fairness violation regrets can be decomposed into the sum of episodic regret (Lemma \ref{lem:11}); (ii) We compute the fairness violation regret when true MDP does not belong to confidence ball, followed by the  regret when true MDP falls into confidence ball (Lemma \ref{lemma:fairfail}, \ref{lemma:goodeventfair}); and (iii) We complete the proof by presenting the reward regret when confidence ball fails to fall into confidence ball/belongs to the confidence ball (Lemma \ref{lemma:failregret}, \ref{lemma:goodevent}).

\subsection{Regret decomposition}
We begin by showing that the cumulative regret can be decomposed into the sum of regrets incurred during each episode. We use reward regret $\Delta^R_k\{\pi^{*,k}\}$ as an example, while the fairness violation regret is essentially the same. For simplicity, we denote 
$c_n^k(s,a):=\sum_{h=1}^{H}\mathds{1}(s_n^{k}(h)=s,a_n^{k}(h)=a)$  as the state-action counts for $(s,a)$ in episode $k$ for arm $n$. Then, under policy $\pi^{*,k}$, we define the regret during episode $k$ as follows: 
\begin{align*}
\Delta^R_k\{\pi^{*,k}\}:=&H\mu^{*}-\sum_{(s,a)}\sum_n c_n^k(s,a)\bar{r}_n(s,a),
\end{align*}
where $\mu^*$ is the average reward per step by the optimal policy. The relation between the total regret $\Delta^R_T\{\pi^{*,k}, \forall k\}$ and the episodic regrets $\Delta^R_k\{\pi^{*,k}\}$ is given as follows:

\subsubsection{Proof of Lemma \ref{lem:11}.}

Let $C_n^k(s,a)$ (note this is different to $c_n^k$, which is the count within episode $k$) be the total number of visits to $(s,a)$ until episode $k$ under policy $\{\pi^{*,k}\}$, and denote $r(\{\pi^{*,k}, \forall k\}, T)$ as the reward until $T$ under policy $\{\pi^{*,k},\forall k\}$.  Using Chernoff-Hoeffding's inequality, we have

\begin{align*}
&\mathbb{P}\Bigg(r(\{\pi^{*,k}, \forall k\}, T)\leq \\
&\qquad \sum\limits_n\sum\limits_{(s,a)} c_n^k(s,a)\bar{r}_n(s,a)- \sqrt{\frac{1}{4}T\log\frac{SANT}{\epsilon}}\Bigg)\\
&\qquad \leq  \exp{\bigg\{-\frac{2 \frac{1}{4} T \log{\frac{SANT}{\epsilon}}}{T (1-0)}\bigg\}} \\
& \qquad =  \exp{\bigg\{-\frac{1}{2}\log{\frac{SANT}{\epsilon}}\bigg\}} \\
& \qquad = \left(\frac{\epsilon}{SANT}\right)^{\frac{1}{2}}.
\end{align*}

Therefore, we obtain 
\begin{align*}
\Delta^R_T\{\pi^{*,k}, \forall k\}&=T\mu^{*}-r(\{\pi^{*,k}, \forall k\}, T)\displaybreak[0]\\
&\leq T\mu^{*}-\sum_k\sum_n\sum_{(s,a)} c_n^k(s,a)\bar{r}_n(s,a) \displaybreak[1]\\
\qquad \qquad & \qquad \qquad + \sqrt{\frac{1}{4}T\log\frac{SANT}{\epsilon}}\displaybreak[2]\\
&=\sum\limits_{k=1}^{K} \Delta^R_k\{\pi^{*,k}\}+ \sqrt{\frac{1}{4}T\log\frac{SANT}{\epsilon}},
\end{align*}
with probability as least $1-(\frac{\epsilon}{{SANT}})^{\frac{1}{2}}$.

\subsection{Fairness violation regret when true MDP does not belong to confidence ball}

The fairness violation regret for arm $n$ until time  $T$ is defined as 
\begin{align*}
    \Delta_T^{n,F} &=T \eta_n-\mathbb{E}_{\pi^*}\Bigg[\sum_{t=1}^Ta_n(t)\Bigg],~\forall n\in\cN. 
\end{align*}

By dividing the regret into episodes, we have the following alternative expression of fairness violation regret
\begin{align*}
    \Delta_T^{n,F} =  \mathbb{E}\bigg[\sum_{k=1}^K \sum_{h=1}^H (\eta_n - a_n^{k,h}(t))\bigg], \forall n\in\cN
\end{align*}
where $a_n^{k,h}$ is the action taken in episode $k$, time frame $h$ for arm $n$. 
According to Lemma \ref{lem:11}, we will decompose the fairness violation regret:
\begin{align}
    \Delta_T^{n,F} &\leq \sum\limits_{k=1}^{K} \Delta_k^{n,F}+ \sqrt{\frac{1}{4}T\log\frac{SANT}{\epsilon}} \nonumber\\
    & = \sum\limits_{k=1}^{K} \Delta_k^{n,F} \mathds{1}((M\in \mathcal{M}_k) + (M\notin \mathcal{M}_k))\nonumber\\
    & \qquad \qquad \qquad + \sqrt{\frac{1}{4}T\log\frac{SANT}{\epsilon}} \label{eq:fairregretdecom}
\end{align}

Next we consider the event that true MDP does not belong to confidence ball, i.e. $M \notin \mathcal{M_k}$. This event happens when
\begin{align*}
    \exists (s,a),n,\text{s.t.}&~ |{P}_{n}^k(s^\prime|s, a) - \hat{P}_{n}^{k}(s^\prime|s, a)| >  \delta_{n}^{k}(s, a).
\end{align*}
The probability of such event is characterized in the following lemma.
\begin{lemma}\label{lemma:confifail}
\text{The probability of confidence ball fails is} $$P(M\notin \mathcal{M}_k) \leq \frac{\epsilon}{kH},$$ \text{where} $$\delta_{n}^{k}(s,a)=\sqrt{\frac{1}{2C_{n}^{k-1}(s,a)}\log{\frac{SANkH}{\epsilon}}}.$$
\end{lemma}
Proof: By Chernoff-Hoeffding's inequality, we have
\begin{align*}
    P(|{P}_{n}^k(s^\prime|s, a) - \hat{P}_{n}^{k}(s^\prime|s, a)| >  \delta_{n}^{k}(s, a)) \leq \frac{\epsilon}{SANkH}.
\end{align*}
Summing over all state-action pairs and different arms, the following bound is given:
\begin{align*}
    P(M\notin \mathcal{M}_k) &\leq \sum_n\sum_{(s,a)}P(|{P}_{n}^k(s^\prime|s, a) - \hat{P}_{n}^{k}(s^\prime|s, a)| \\
    &\leq \frac{\epsilon}{kH}.
\end{align*}

\subsubsection{Proof of Lemma \ref{lemma:fairfail}.}
For the case that confidence ball fails, we can bound regret as following:
\begin{align*}
    \sum\limits_{k=1}^{K} \Delta_k^{n,F} \mathds{1} (M\notin \mathcal{M}_k)&\leq \sum_{k=1}^K \sum_{h=1}^H \eta_n \mathds{1} (M\notin \mathcal{M}_k)\\
    &\leq \sum_{k=1}^K H \eta_n \frac{\epsilon}{kH}\\
    &\leq \eta_n\epsilon \sum_{k=1}^K \frac{1}{k}\\
    &\leq \eta_n\epsilon \log{K} = \frac{1}{2} \eta_n  \epsilon \log{T} 
\end{align*}
with probability at least $1 - (\frac{\epsilon}{SANT})^{\frac{1}{2}} $, where the first inequality is because $a_{n}^{k,h}(t)\geq 0$, second is due to Lemma \ref{lemma:confifail}, last equality holds by setting $K=H=\sqrt{T}.$ 

\subsection{Fairness violation regret when true MDP belongs to confidence ball}

Next, we will discuss the dominant part of fairness violation regret $\sum\limits_{k=1}^{K} \Delta_k^{n,F} \mathds{1}(M\in \mathcal{M}_k)$, which is when the true MDP belongs to the confidence ball.

\subsubsection{Proof of Lemma \ref{lemma:goodeventfair}.}

Define $\overline{F}_n(\pi^k,p) = \frac{1}{T} \lim_{T\rightarrow\infty} (\sum_{t=1}^T a_n(t)|\pi^k,p)$ as the long term average fairness variable under policy $\pi^k$ for arm $n$ with MDP that has true transition probability matrix $p$ (i.e. this is the long term average fairness variable if we apply the policy $\pi^k$ to all time slots with transition $p$), and define $\overline{F}_n(\pi^k,\theta) = \frac{1}{T} \lim_{T\rightarrow\infty} (\sum_{t=1}^T a_n(t)|\pi^k,\theta)$ as the fairness variable under policy $\pi$ in episode $k$ for arm $n$ with MDP whose transition matrix $\theta$ belongs in the confidence ball. Before we moving forward, we need to introduce the \emph{unichain} assumption and a corollary from \cite{mitrophanov2005sensitivity} (Corollary 3.1): 
\begin{assumption}\label{assum:unichain}
Assume that all MDPs in our \frmab is unichain, i.e., for MDP with transition $p$ with a stationary policy $\pi$, there exists positive constants $\rho>0$ and $C< \infty$, s.t.
\begin{align*}
    ||P_{\pi,p,s}^t - P_{\pi,p}||_{TV}\leq C \rho^t, t= 1,2,... 
\end{align*}
where $P_{\pi,p,s}^t$ is the transition distribution after $t$ time steps starting from state $s$, $P_{\pi,p}$ is the stationary distribution for policy $\pi$ under transition $p$.
\end{assumption}
\begin{corollary}\label{coro:sensitivity}\citep{mitrophanov2005sensitivity}
Define $P^{(1)}_{\hat{p},s} (s)$ as the transition starting from state $s $ applying policy based on $\hat{p}$, which is the transition estimation based on state, action count (and $P^{(1)}_{p,s}$ from true transition $p$, $P^{(1)}_{\theta,s}$ from $\theta$, respectively). By definition, $\overline{F}_n(\pi^k,p)$ is an invariant measure for chain with transition $P^{(1)}_{p,s}$, $\overline{F}_n(\pi^k,\theta)$ is a invariant measure for chain with transition $P^{(1)}_{\theta,s}$, 
\begin{align*}
    ||\overline{F}_n\!(\pi^k\!,\!p\!)\! -\! \overline{F}_n(\!\pi^k\!,\!\theta\!)||_{TV}\! \leq\! \bigg(\! \hat{n}\! +\! \frac{C\rho^{\hat{n}}}{1-\rho}\! \bigg)\!||P^{(1)}_{p,s}\! -\! P^{(1)}_{\theta,s}||_{TV},
\end{align*}
where $\hat{n} = \lceil \log_\rho{C^{-1}} \rceil$ is a constant. 
\end{corollary}
With Assumption \ref{assum:unichain}, the fairness violation regret when confidence ball fails for episode $k$ is characterized in the following lemma.
\begin{lemma}\label{lemma:consrewrite}
The event that the true MDP belongs to the confidence ball in fairness violation for episode $k$ is upper bounded by 
$$\Delta_k^{n,F} \mathds{1}(M\in \mathcal{M}_k)\leq \mathbb{E} \bigg[ H\eta_n - \overline{F}_n(\pi_n^k,p) \bigg]  + \frac{C}{1-\rho}. $$
\end{lemma}
\begin{proof}
For the fairness violation regret in episode $k$, we have 
\begin{align*}
    \Delta_k^{n,F}\mathds{1}(M\in \mathcal{M}_k) &= \mathbb{E} \bigg[ H\eta_n - \sum_{h=1}^H a_n^{k,h}(t)\bigg]\\
    &=  H\mathbb{E} \bigg[ \eta_n - \overline{F}_n(\pi_n^k,p) \bigg] \\
    & \qquad  + \mathbb{E} \bigg[ \sum_{h=1}^H ( \overline{F}_n(\pi^k,p) - a_n^{k,h}(t))\bigg]\\
    &\leq H\mathbb{E} \bigg[ \eta_n - \overline{F}_n(\pi_n^k,p) \bigg]  + \frac{C}{1-\rho},
\end{align*}
where the last inequality holds because of the unichain assumption if we consider $\overline{F}_n(\pi_n^k,p)$ as the true transition distribution and $\sum_{h=1}^H a_n^{k,h}(t)$ as the transition distribution after $H$ time steps. 
\end{proof}

With Assumption \ref{assum:unichain} and Corollary \ref{coro:sensitivity}, we have
\begin{lemma}\label{lemma:consrange}
Under policy $\pi$, for $\theta \in \mathcal{M}_k$, 
\begin{align*}
    &|\overline{F}_n(\pi^k,p) - \overline{F}_n(\pi^k,\theta)| \leq\\
&\qquad \qquad \qquad2 \bigg( \hat{n} + \frac{C\rho^{\hat{n}}}{1-\rho} \bigg)\max_{s} \sum_a \pi(s,a) \delta_n^k(s,a).
\end{align*} 
\end{lemma}
\begin{proof}
    Since both the true transition $p$ and $\theta$ are in confidence ball, we obtain
\begin{align*}
    &||P^{(1)}_{\hat{p},s} - P^{(1)}_{p,s}||_{TV}\leq \sum_a \pi(s,a) \delta_n^k(s,a),\displaybreak[0]\\
    &||P^{(1)}_{\hat{p},s} - P^{(1)}_{\theta,s}||_{TV} \leq  \sum_a \pi(s,a) \delta_n^k(s,a),
\end{align*}
with triangle inequality, we immediately obtain 
\begin{align*}
    ||P^{(1)}_{p,s} - P^{(1)}_{\theta,s}||_{TV} \leq 2  \sum_a \pi(s,a) \delta_n^k(s,a).
\end{align*}
Combining with Corollary~\ref{coro:sensitivity}, we have 
\begin{align*}
    &||\overline{F}_n(\pi^k,p) - \overline{F}_n(\pi^k,\theta)||_{TV} \\
    &\qquad \leq  2 \bigg( \hat{n} + \frac{C\rho^{\hat{n}}}{1-\rho} \bigg)  \sum_a \pi(s,a) \delta_n^k(s,a) \\
    &\qquad \leq 2 \bigg( \hat{n} + \frac{C\rho^{\hat{n}}}{1-\rho} \bigg)\max_{s} \sum_a \pi(s,a) \delta_n^k(s,a)
\end{align*}
The result follows since $|\overline{F}_n(\pi^k,p) - \overline{F}_n(\pi^k,\theta)|$ is the value difference and $||\overline{F}_n(\pi^k,p) - \overline{F}_n(\pi^k,\theta)||_{TV}$ is in the total variation norm.

\end{proof}

\begin{lemma}\label{lemma:consbound}
Define $\beta_n^k(\pi_k) := 2 \bigg( \hat{n} + \frac{C\rho^{\hat{n}}}{1-\rho} \bigg)\max_{s} \sum_a \pi(s,a) \delta_n^k(s,a)$. Then for any feasible policy $\pi^k$, we have  $$\overline{F}_n(\pi^k,p)\geq \eta_n - \beta_n^k(\pi_n^k).$$
\end{lemma}
\begin{proof}
Assume there exists $\overline{F}_n(\pi^k,p) < \eta_n - \beta_n^k(\pi_n^k)$. Since $|\overline{F}_n(\pi^k,p) - \overline{F}_n(\pi^k,\theta)| \leq \beta_n^k(\pi_n^k)$, we have $\overline{F}_n(\pi^k,\theta) < \eta_n$ for any policy $\pi^k$. This means under policy $\pi^k$, any transition probability $\theta$ within confidence ball will not fulfill the fairness constraint. This contradicts to $\pi^k$ for true transition $p$ is feasible. 
\end{proof}

Next we introduce an auxiliary lemma, i.e., Lemma 3 in \cite{singh2020learning}.
\begin{lemma}[\cite{singh2020learning}]\label{lemma:mixing}
Define $T_{\text{mix}}^n$ as the mixing time for MDP correponding to arm $n$ as $$\mathbb{E}[{c_n^k(s,a)}]\leq \bigg\lfloor \frac{H}{2T_{\text{mix}}^n} \bigg\rfloor \frac{\pi_n^k(a|s)}{2}. $$
\end{lemma}
Using Lemmas~\ref{lemma:consrange}, ~\ref{lemma:consbound}, ~\ref{lemma:consrewrite} and~\ref{lemma:mixing}, we can further bound the fairness violation regret in episode $k$ as follows
\begin{align*}
    &\Delta_k^{n,F}\mathds{1}(M\in \mathcal{M}_k) \\
    &\leq  \beta_n^k(\pi_k) H + \frac{C}{1-\rho}\displaybreak[0] \\
   & \leq \sum_{(s,a)} H \bigg( \hat{n} + \frac{C\rho^{\hat{n}}}{1-\rho} \bigg) \pi(s,a) \delta_n^k(s,a) + \frac{C}{1-\rho}\displaybreak[0]\\
  & = 4 T_{\text{Mix}}^n \bigg( \hat{n} + \frac{C\rho^{\hat{n}}}{1-\rho} \bigg)\sum_{(s,a)} \frac{H}{2T_{\text{Mix}}} \frac{1}{2} \pi(s,a) \delta_n^k(s,a)+ \frac{C}{1-\rho}\displaybreak[0]\\ 
  & = 4 T_{\text{Mix}}^n \bigg( \hat{n} + \frac{C\rho^{\hat{n}}}{1-\rho} \bigg)\sum_{(s,a)} \mathbb{E} [c_n^k(s)] \pi(s,a) \delta_n^k(s,a) \displaybreak[0]\\
  & + 4 T_{\text{Mix}}^n\! \bigg(\! \hat{n}\! +\! \frac{\!C\!\rho^{\hat{n}}}{1-\rho}\! \bigg)\!\sum_{(s,a)}\! (\!\frac{H}{2T_{\text{Mix}}} \!\frac{1}{2}\! -\!\mathbb{E}\! [c_n^k(s)])  \pi(s,a) \delta_n^k(s,a) \displaybreak[0]\\
  & + \frac{C}{1-\rho}\displaybreak[0]\\
  & \overset{(a)}\leq 4 T_{\text{Mix}}^n \bigg(\! \hat{n}\! +\! \frac{C\rho^{\hat{n}}}{1-\rho} \bigg)\sum_{(s,a)} \mathbb{E}[c_n^k(s)] \pi(s,a) \delta_n^k(s,a)\!+\! \frac{C}{1-\rho}\displaybreak[0]\\
  & = 4 T_{\text{Mix}}^n\! \bigg(\! \hat{n}\! +\! \frac{C\rho^{\hat{n}}}{1-\rho}\! \bigg)\!\sum_{(s,a)}\! \mathbb{E} [c_n^k(s,a)]  \delta_n^k(s,a)\!+\! \frac{C}{1-\rho} \displaybreak[0]\\
  & = 4 T_{\text{Mix}}^n\! \bigg(\! \hat{n}\! +\! \frac{C\rho^{\hat{n}}}{1-\rho}\! \bigg)\!\sum_{(s,a)} \!\mathbb{E} [c_n^k(s,a)] \sqrt{\frac{\log{\frac{SANkH}{\epsilon}}}{2C_{n}^{k-1}(s,a)}}\!+\! \frac{C}{1-\rho}\displaybreak[0]\\
  & = 4 T_{\text{Mix}}^n \bigg( \hat{n} + \frac{C\rho^{\hat{n}}}{1-\rho} \bigg)\underset{\text{Term 1}}{\underbrace{\sum_{(s,a)} \sqrt{\log{\frac{SANkH}{\epsilon}}} \frac{c_n^k(s,a) }{\sqrt{2C_{n}^{k-1}(s,a)}}}}\displaybreak[0]\\
  & +\! 4\! T_{\text{Mix}}^n\! \bigg(\! \hat{n} +\! \frac{C\rho^{\hat{n}}}{1-\rho}\! \bigg)\!\underset{\text{Term 2}}{\underbrace{\!\sum_{(s,a)}\! \sqrt{\log{\!\frac{\!S\!A\!N\!k\!H\!}{\epsilon}}} \frac{\mathbb{E} [c_n^k(s,a)] - c_n^k(s,a)}{\sqrt{2C_{n}^{k-1}(s,a)}}}}\\\displaybreak[0]
  &+ \frac{C}{1-\rho},
\end{align*}
where (a) is because of Lemma~\ref{lemma:mixing}.

In order to bound Term 1, we leverage Lemma 19 in \cite{jaksch2010near}. However, there exists a major difference in the settings between UCRL2 \cite{jaksch2010near} and \fairucrl.  In \emph{UCRL2}, the episode stopping criteria implies $c_n^k(s,a) \geq C_n^{k-1}(s,a)$, while we use a fixed episode length $H$ in \fairucrl. 
\begin{lemma}\label{lemma:coverc}
    For any sequence of numbers $w_1,w_2,...,w_n$ with $0\leq w_k$, define $W_{k}:=\sum_{i=1}^k w_i$,
    \begin{align*}
        \sum_{k=1}^n \frac{w_k}{\sqrt{W_{k-1}}} \leq (\sqrt{2} +1)\sqrt{W_n}.
    \end{align*}
\end{lemma}
\begin{proof}
The proof follows by induction. 
When $n=1$, it is true as $1\leq\sqrt{2}+1$.
Assume for all $k\leq n-1$, the inequality holds, then we have the following:
\begin{align*}
    &\sum_{k=1}^n \frac{w_k}{\sqrt{W_{k-1}}}= \sum_{k=1}^{n-1} \frac{w_k}{\sqrt{W_{k-1}}} + \frac{w_n}{\sqrt{W_{n}}} \displaybreak[0]\\
    &\leq (\sqrt{2} +1) \sqrt{W_{n-1}} + \frac{w_n}{\sqrt{W_{n}}} \displaybreak[0]\\
    &=\sqrt{{(\sqrt{2}+1)}^2 W_{n-1} + 2(\sqrt{2}+1)w_n \sqrt{\frac{W_{n-1}}{W_n}} + \frac{{w_n}^2}{W_n}} \displaybreak[0]\\
    &\leq \sqrt{{(\sqrt{2}+1)}^2 W_{n-1} + 2(\sqrt{2}+1)w_n \sqrt{\frac{W_{n-1}}{W_{n-1}}} + \frac{{w_n}W_n}{W_n}} \displaybreak[0]\\
    & = \sqrt{{(\sqrt{2}+1)}^2 W_{n-1} + (2(\sqrt{2}+1)+1)w_n} \displaybreak[0]\\
    & = (\sqrt{2}+1)\sqrt{(W_{n-1}+ w_n)}\displaybreak[0]\\
    & = (\sqrt{2}+1)\sqrt{W_n}.
\end{align*}
\end{proof}
Using Lemma \ref{lemma:coverc}, we have
\begin{align}
&\sum_{k=1}^K \sum_{(s,a)} \sum_n \frac{c_n^k(s,a)}{C_n^{k-1}(s,a)} \nonumber\\
&\leq (\sqrt{2}+1) \sum_{(s,a)} \sum_n\sqrt{C_n^k(s,a)}\nonumber\\
&\overset{(a)}{\leq} (\sqrt{2}+1)  \sqrt{SANT} , \label{eq:converc}
\end{align}
where (a) is from Jensen's inequality.
Then we can bound Term 1 with $(\sqrt{2}+1)\sqrt{SANT} \sqrt{\log{\frac{SANkH}{\epsilon}}}.$ For Term 2, since $\mathbb{E} [c_n^k(s,a)] - c_n^k(s,a) \leq H, C_{n}^{k-1}(s,a) \geq 1$, we have $ \frac{\mathbb{E} [c_n^k(s,a)] - c_n^k(s,a)}{\sqrt{2C_{n}^{k-1}(s,a)}} \leq \frac{H}{\sqrt{2}}.$

Taking $\sum_{k=1}^K  \sum_{(s,a)} \frac{\mathbb{E} [c_n^k(s,a)] - c_n^k(s,a)}{\sqrt{2C_{n}^{k-1}(s,a)}}$ as the sum of martingale sequential difference where the sequence difference is $X(k)- X(k-1) =  \sum_{(s,a)}  \frac{\mathbb{E} [c_n^k(s,a)] - c_n^k(s,a)}{\sqrt{2C_{n}^{k-1}(s,a)}}$. According to Azuma-Hoeffding’s inequality, by setting $K = H = \sqrt{T}$, we have 
\begin{align*}
    &P\!\bigg(\!\sum_{k=1}^K \!  \sum_{(s,a)}\! \frac{\mathbb{E} [c_n^k(s,a)] - c_n^k(s,a)}{\sqrt{2C_{n}^{k-1}(s,a)}}\! \geq\! \sqrt{T}\! \sqrt{\frac{1}{4}\log{\frac{SANT}{\epsilon}}}\bigg) \\
    &\leq \exp{\frac{-T\frac{1}{4}\log{{\frac{SANT}{\epsilon}}}}{\frac{1}{2} T^{3/2}}}\! =\! \exp{\frac{\log{\frac{\epsilon}{SANT}}^{1/2}}{\ T^{1/2}}}\\
    &\leq\! (\frac{\epsilon}{SANT})^{1/2}.
\end{align*}
This means with probability of $1 - (\frac{\epsilon}{SANT})^{1/2}$, the summation of $k$ episodes of Term 2 can be bounded as:
\begin{align*}
    T^{1/2} \sqrt{\frac{1}{4}\log{\frac{SANT}{\epsilon}}} \sqrt{\log{\frac{SANT}{\epsilon}}} = \frac{1}{2} \sqrt{T} \log{\frac{SANT}{\epsilon}}
\end{align*}

With probability $1-(\frac{\epsilon}{SANT})^{1/2}$ we have the dominant term (true MDP belongs to the confidence ball) in fairness violation regret as
\begin{align}
    &\Delta_T^{n,F}\mathds{1}(M\in \mathcal{M}_k) \leq 4T_{\text{Mix}}^n \bigg( \hat{n} + \frac{C\rho^{\hat{n}}}{1-\rho} \bigg)\nonumber\\
    &\bigg(\sqrt{\log{\frac{SANT}{\epsilon}}} (\sqrt{2}+1)\sqrt{SANT}+\frac{1}{2} \sqrt{T} \log{\frac{SANT}{\epsilon}} \bigg)\nonumber\\
    &+ \sqrt{T}\frac{C}{1-\rho}. \label{eq:fairsuccess}
\end{align}
\subsection{Total fairness violation regert}

Substituting \eqref{eq:fairsuccess}, Lemma \ref{lemma:fairfail} into \eqref{eq:fairregretdecom}, and for simplicity, define $C_0 := 4 (\sqrt{2} +1)\bigg( \hat{n} + \frac{C\rho^{\hat{n}}}{1-\rho} \bigg),$ we obtain the final bound for fairness violation regret:
\begin{align*}
    \Delta_T^{n,F} &\leq  \frac{1}{2} \eta_n  \epsilon \log{T} + \frac{C}{1-\rho} \sqrt{T} + \sqrt{\frac{1}{4}T\log\frac{SANT}{\epsilon}} \nonumber\\
    &+ 4T_{\text{Mix}}^n \bigg( \hat{n} + \frac{C\rho^{\hat{n}}}{1-\rho} \bigg)\\
    &\bigg(\sqrt{\log{\frac{SANT}{\epsilon}}}\! (\!\sqrt{2}\!+\!1\!)\!\sqrt{SANT}\!+\!\frac{1}{2}\! \sqrt{T} \log{\frac{SANT}{\epsilon}} \bigg) \nonumber\\
    & = \tilde{\mathcal{O}} \Bigg(  \eta_n  \epsilon \log{T}   +   C_0 T_{\text{Mix}}^n    \sqrt{SANT}\log\!{\frac{SANT}{\epsilon}}  \Bigg), 
\end{align*}
with probability at least $(1 - (\frac{\epsilon}{SANT})^{\frac{1}{2}})^2 $.

\subsection{Reward regret when the confidence ball fails} \label{sec:confidencefail}

Using Lemma \ref{lemma:confifail}, we can show the regret bound for failing confidence ball.

\subsubsection{Proof of Lemma \ref{lemma:failregret}}

By Lemma \ref{lemma:confifail}, we have 
\begin{align*}
    \sum_{k=1}^K \Delta_k \mathds{1}(M\notin \mathcal{M}_k) &\leq \sum_{k=1}^K  \sum_n \sum_{(s,a)} c_n^k(s,a) \mathds{1}(M\notin \mathcal{M}_k)\\
    &= \sum_{k=1}^K HB \mathds{1}(M\notin \mathcal{M}_k)\\
    & \leq \sum_{k=1}^K HB \frac{\epsilon}{kH}= B\epsilon \sum_{k=1}^K \frac{1}{k}\\
    &\leq B\epsilon \sum_{1=2}^T \frac{1}{(t-1)} \\&
    \leq B\epsilon \int_{t=1}^ T \frac{1}{t} dt  \leq B\epsilon \log{T}.
\end{align*}

\subsection{Reward regret when true MDP belongs to the confidence ball} \label{sec:confidence belongs}
\subsubsection{Proof of Lemma \ref{lemma:goodevent}}

Given $M\in \mathcal{M}_k$, we bound the regret in episode $k$ as follows:
\begin{align} \nonumber
 &\Delta^R_k\{\pi^{*,k}\} \mathds{1} (M\in \mathcal{M}_k)\nonumber\\
 &= H \mu^* - \sum_{(s,a)}\sum_n c_n^k(s,a)\bar{r}_n(s,a)\nonumber \displaybreak[0]\\
& \overset{(a)} =  \sum_{(s,a)}\sum_n c_n^k(s,a) \frac{\mu^*}{B}  - \sum_{(s,a)}\sum_n c_n^k(s,a)\bar{r}_n(s,a)\nonumber\displaybreak[0]\\
&= \sum_{(s,a)} \sum_n  c_{n}^k(s,a)(\mu^* /B - \bar{r}_n(s,a)) \nonumber\displaybreak[0] \\
&= \sum_{(s,a)} \sum_n  c_{n}^k(s,a)(\mu^* /B - \tilde{r}_n(s,a)) \nonumber \\
&+ \sum_{(s,a)} \sum_n  c_{n}^k(s,a) (\tilde{r}_n(s,a) - \bar{r}_n(s,a))\nonumber\displaybreak[0] \\
&\overset{(b)}\leq \sum_{(s,a)} \sum_n  c_{n}^k(s,a)(\mu^* /B - \tilde{r}_n(s,a)) \nonumber \\
&+ \sum_{(s,a)} \sum_n c_{n}^k(s,a) 2 \sqrt{\frac{1}{2C_{n}^{k-1}(s,a)}\log{\frac{SANkH}{\epsilon}}} \nonumber\displaybreak[0]\\
&\overset{(c)}\leq  \sqrt{2\log{\frac{SANT}{\epsilon}}}   \sum_{(s,a)}  \sum_n \frac{c_{n}^k(s,a)}{\sqrt{C_{n}^{k-1}(s,a)}}, \label{eq:Regret_goodevent}
\end{align}
where (a) holds because $\sum_{(s,a)} \sum_n c_n^k (s,a) = HB$; (b) is due to $M\in \mathcal{M}_k$; and (c) holds due to the fact that for any episode $k$, the optimistic average reward $\tilde{r}_n (s,a)$ of the optimistically chosen policy ${\tilde{\pi}}_k$ is larger than the true optimal average reward $\mu^*$.\\

Substituting \eqref{eq:converc} into \eqref{eq:Regret_goodevent} leads to the result in Lemma \ref{lemma:goodevent}.

\subsection{Total reward regret} 

Combining all bounds in Lemmas \ref{lem:11}, \ref{lemma:failregret} and \ref{lemma:goodevent}, we can obtain the total regret as:
\begin{align*}
    &\Delta^R_T\{\pi^{*,k}, \forall k\}\\
    &=\sum\limits_{k=1}^{K} \Delta^R_k\{\pi^{*,k}\}+ \sqrt{\frac{1}{4}T\log\frac{SANT}{\epsilon}}\displaybreak[0]\\
    &= \sum_{k=1}^K \bigg( \Delta^R_k\{\pi^{*,k}\} (\mathds{1}(M\in \mathcal{M}_k) + \mathds{1}(M\notin \mathcal{M}_k))\bigg) \\
    &+ \sqrt{\frac{1}{4}T\log\frac{SANT}{\epsilon}} \displaybreak[1]\\
    &\leq B\epsilon\log{T} + (\sqrt{2}+2) \sqrt{\log{\frac{SANT}{\epsilon}}}   \sqrt{SANT}\\
    &+ \sqrt{\frac{1}{4}T\log\frac{SANT}{\epsilon}}\displaybreak[2]\\
    &=\tilde{\mathcal{O}}\Bigg(  B\epsilon\log{T} + (\sqrt{2}+2) \sqrt{SANT} \sqrt{\log{\frac{SANT}{\epsilon}}} \Bigg).
\end{align*}

\begin{algorithm}[t]
\caption{\gfairucrl}
\label{alg:bfUCRL}
\begin{algorithmic}[1]
	\State {\bfseries Require:} Initialize $C_n^{0}(s,a)=0,$ and  $\hat{P}_n^{0}(s^\prime|s,a)=1/S$, $\forall n\in\cN, s,s^\prime\in\cS, a\in\cA$.	
	\For{$k=1,2,\cdots,K$}	
 \State $//** $\textit{Greedy Exploration for Fairness}$ **//$
   \State In the first $\lceil \frac{\sum_{n=1}^N H\eta_n}{B}\rceil$ decision epochs, select each arm to meet the fairness constraint requirement;
	\State $//** $\textit{Optimistic Planning}$ **//$
	\State Construct the set of plausible MDPs $\cM^k$ as in \eqref{eq:plausible2}; 
	\State Relaxed the instantaneous activation constraint in \frmab($\tilde{P}_n^k, \tilde{r}_n^k, \forall n$) to be ``long-term activation constraint'', and transform the relaxed problem into 
 $\textbf{ELP}^1(\mathcal{M}^k, z^k)$~(\ref{eq1:UCB_extended-obj})-(\ref{eq1:UCB_extended-cons6});
	\State $//** $\textit{Policy Execution}$ **//$
	\State Establish the \fair index policy $\pi^{k,*}$ on top of the solutions to the extended LP and execute it.  
	\EndFor
\end{algorithmic}
\end{algorithm}

\section{\gfairucrl and Regret Analysis}
Our \fairucrl strictly meets the instantaneous activation constraint, i.e., it operates in a way that exactly $B$ arms are activated at each decision epoch, which is guaranteed by the index policy.  Though \fairucrl provably achieves sublinear bounds on both reward and fairness violation regrets, some applications (e.g., healthcare) may have a stricter requirement on   fairness. Now, we show that it is possible to design an episodic RL algorithm with no fairness violation.

\subsection{The \gfairucrl Algorithm}

\gfairucrl proceeds in episodes as summarized in Algorithm~\ref{alg:bfUCRL}. Different from \fairucrl, in each episode, our \gfairucrl starts with \textbf{a greedy exploration} to first ensure that the fairness requirement $\eta_n$ for each arm $n$ is satisfied. Specifically, we guarantee this in a greedy manner, i.e., at the beginning of each episode $\tau_k$, \gfairucrl randomly pulls an arm to force each arm $n$ to be pulled $H\eta_n$ times. This greedy exploration will take $\lceil \frac{\sum_{n=1}^N H\eta_n}{B}\rceil$ decision epochs in total. Similar to \fairucrl, our \gfairucrl follows the phases of optimistic planning and policy execution after the greedy exploration. The major difference is that \gfairucrl will take the samples from the greedy exploration into account when constructing the set of plausible MDPs as in~(\ref{eq:plausible21}), but no longer need the fairness constraint when solving the extended LP in~(\ref{eq:ELP}). The new extended LP over $z^k:=\{z_{n}^k(s, a, s^\prime), \forall n\in\cN\}$ is then given as $\textbf{ELP}^1(\mathcal{M}^k, z^k):$
\begin{align}
\min_{z^k}&~\sum_{n=1}^{N}\sum_{(s,a,s^\prime)} z_{n}^k(s,a, s^\prime)\tilde{r}_n^k(s,a) \label{eq1:UCB_extended-obj}\\
\mbox{ s.t. } &~\sum_{n=1}^{N}\sum_{s\in\cS}\sum_{s^\prime\in\cS} z_{n}^k(s,1,s^\prime) \le B , \displaybreak[0]\label{eq1:UCB_extended-cons1}\\
&~\sum_{a\in\cA}\sum_{s^\prime\in\cS} z_{n}^k(s,a,s^\prime)=\sum_{s^\prime\in\cS}\sum_{a^\prime\in\cA}z_{n}^k(s^\prime, a^\prime, s), \displaybreak[2] \label{eq1:UCB_extended-cons3}\\ 
&~\sum_{s\in\cS}\sum_{a\in\cA}\sum_{s^\prime\in\cS} z_{n}^k(s,a, s^\prime)=1,~\forall n\in\cN, \displaybreak[3]\label{eq1:UCB_extended-cons4}\\
&~\frac{z_{n}^k(s,a,s^\prime)}{\sum_y z_{n}^k(s,a,y)}\!-\!(\hat{P}_{n}^k(s^\prime|s,a)\!+\!\delta_{n}^k(s,a))\leq 0, \label{eq1:UCB_extended-cons5}\\ 
&\!-\!\frac{z_{n}^k(s,a,s^\prime)}{\sum_y z_{n}^k(s,a,y)}\!+\!(\hat{P}_{n}^k(s^\prime|s,a)\!-\!\delta_{n}^k(s,a))\leq0, \label{eq1:UCB_extended-cons6} 
\end{align}

\subsection{Regret Analysis of \gfairucrl}\label{sec:bffairregret}
The next theoretical contribution is the regret analysis for \gfairucrl.
\begin{theorem}\label{thm:fairregret}
With the same size of confidence intervals $\delta_{n}^{k}(s,a)$ as in Theorem \ref{thm:regret}, and
with probability at least $ 1-(\frac{\epsilon}{{SANT}})^{\frac{1}{2}}$, \gfairucrl achieves the reward regret as
\begin{align*}
\Delta_T^{R} \leq \tilde{\mathcal{O}} \Bigg(& \Big\lceil \frac{\sum_n \eta_n}{B}\Big\rceil T  + B\epsilon\log{T}\\
&+ (\sqrt{2}+2) \sqrt{SANT} \sqrt{\log{\frac{SANT}{\epsilon}}} \Bigg).
\end{align*}
and guarantees 0 fairness violation regret.
\end{theorem}

\subsection{Proof of Theorem \ref{thm:fairregret}}

Similar to reward regret proof for \fairucrl, we can divide the regret to episodes, which contains the first $\lceil\frac{\sum_n H\eta_n}{B}\rceil$ time steps and remaining $\lceil H-\frac{\sum_n H \eta_n}{B} \rceil$ for policy execution. The reward regret satisfies: 
\begin{align*}
\Delta^R_T\{\pi^{*,k}, \forall k\}:&=T\mu^{*}-r(\{\pi^{*,k}, \forall k\}, T)\nonumber\\
&\leq \sum\limits_{k=1}^{K} \Delta^R_k\{\pi^{*,k}\}+ \sqrt{\frac{1}{4}T\log\frac{SANT}{\epsilon}}\nonumber\\
&=\sum_{k=1}^K ( \Delta^R_{k,bf}\{\pi^{*,k}\} + \Delta^R_{k,pe}\{\pi^{*,k}\} )\\
&+\sqrt{\frac{1}{4}T\log\frac{SANT}{\epsilon}}\nonumber\\
\end{align*}
While the latter term is the same, we will have following term for brutal force fairness:
\begin{align*}
\Delta^R_{k,bf}\{\pi^{*,k}\}=&\lceil\frac{\sum_n H\eta_n}{B}\rceil\mu^{*}-\sum_{(s,a)}\sum_n c_{n,bf}^k(s,a)\bar{r}_n(s,a) \\
&\leq \lceil\frac{\sum_n H\eta_n}{B}\rceil
\end{align*}
where $c_{n,bf}^k(s,a)$ is the count of state-action pair for arm $n$ in the brutal force time period. This is because for the first $\lceil\frac{\sum_n H\eta_n}{B}\rceil$ time steps, the upper bound for optimal per step reward and randomized policy reward will be 1. Since we set $K=H=\sqrt{T}$, the sum of $K$ episodes of 
\begin{align*}
    &\Delta^R_T\{\pi^{*,k}, \forall k\}\leq  \lceil \frac{\sum_n \eta_n}{B}\rceil \sqrt{T} + B\epsilon\log{T} \\
    &+ (\sqrt{2}+2) \sqrt{\log{\frac{SANT}{\epsilon}}}   \sqrt{SANT}+ \sqrt{\frac{1}{4}\log\frac{SANT}{\epsilon}}\bigg) \\
    &= \tilde{\mathcal{O}} \Bigg( \lceil \frac{\sum_n \eta_n}{B}\rceil \sqrt{T}  + B\epsilon\log{T}+    \lceil \frac{\sum_n \eta_n}{B}\rceil \sqrt{T}  \\
    &+ (\sqrt{2}+2) \sqrt{SANT} \sqrt{\log{\frac{SANT}{\epsilon}}} \Bigg).
\end{align*}

\section{\frmab relaxation and linear programming}\label{sec:apprelaxLP}
We relax the instantaneous activation constraint~(\ref{eq:FRMAB-constraint1}) to be satisfied on long term average \citep{whittle1988restless}, which leads to the following ``relaxed \frmab'' (\refrmab) problem: 
\begin{align}
  \textbf{\refrmab}:\max_{\pi}&\liminf_{T \rightarrow \infty} \frac{1}{T}\mathbb{E}\Bigg[ \sum_{t=1}^T \sum_{n=1}^N r_n(t) \Bigg] \displaybreak[0]\label{eq:relaxed-FRMAB}\\
    \text{subject to}& \limsup_{T \rightarrow \infty}\! \frac{1}{T} \mathbb{E}\Bigg[\!\sum_{t=1}^T\!\sum_{n=1}^N a_{n}(t)\! \Bigg]\!\!\leq\! B,\displaybreak[1] \label{eq:relaxed-FRMAB-constraint1} \\
    & \liminf_{T \rightarrow \infty}\! \frac{1}{T} \mathbb{E}\Bigg[\!\sum_{t=1}^T a_n(t)\!\Bigg]\!\!\geq\!\eta_n, \forall n. \label{eq:relaxed-FRMAB-constraint2} 
\end{align}
One can easily check that the optimal reward achieved by \refrmab in~(\ref{eq:relaxed-FRMAB})-(\ref{eq:relaxed-FRMAB-constraint2}) is an upper bound of that achieved by \frmab in~(\ref{eq:FRMAB-obj})-(\ref{eq:FRMAB-constraint2}) due to the relaxation.  More importantly, \refrmab can be equivalently reduced to a LP in occupancy measures \citep{altman1999constrained}. 
We then equivalently rewrite the resultant problem for episode $k$ into the following linear programming (LP): 
\begin{align}
    \max_{\zeta_n\in\Omega_\pi}~&  \sum_{n=1}^N \sum_{s\in\cS}\sum_{a\in\cA} \zeta^k_{n}(s,a) {\tilde{r}}_n^k(s,a)\displaybreak[0]\label{eq:lp-obj}\\
    \text{s.t.}~&\sum_{n=1}^N \sum_{s\in\cS} \zeta^k_{n}(s,1) \leq B,\displaybreak[1]\label{eq:lp-constraint1} \\
    &\sum_{s\in\cS} \zeta^k_n(s,1) \geq \eta_n, \forall n \in\cN,\displaybreak[2]\label{eq:lp-constraint2}  \\
    &\sum_{s^\prime\in\cS}\sum_{a\in\cA} \zeta^k_{n}(s,a)\tilde{P}_n^k(s^\prime|s,a) \nonumber\\
    & =\sum_{s^\prime\in\cS}\sum_{a^\prime\in\cA}\zeta^k_{n}(s^\prime,a^\prime)\tilde{P}_n^k(s|s^\prime,a^\prime), 
\forall n\in\cN,\displaybreak[3]\label{eq:lp-constraint3} \\
    &\sum_{s\in\cS}\sum_{a\in\cA}\zeta^k_n(s,a)=1,\forall n\in\cN. \label{eq:lp-constraint4} 
\end{align}
One can arrive at the the above LP via replacing all random variables in \frmab($\tilde{P}_n, \tilde{r}_n, \forall n$) with the long-term activation constraint with the occupancy measure corresponding to each arm $n$ \citep{altman1999constrained}.  Specifically, in episode $k$, the occupancy measure $\Omega^k_{\pi}$ of a stationary policy $\pi$ for the infinite-horizon MDP is defined as the expected average number of visits to each state-action pair $(s,a)$, i.e.,
\begin{align}\label{eq:occupancy-measuresupp}
    &\Omega^k_\pi =\Bigg\{ \zeta^k_n(s,a)\triangleq \lim_{H\rightarrow \infty}\frac{1}{H} \mathbb{E}_\pi \nonumber\\
    & \Bigg[ \sum_{h=1}^H \mathds{1} (s_n(h)= s, a_n(h)=a)\Bigg] \Bigg\vert \forall n\in\cN, s\in \cS, a\in \cA \Bigg\},
\end{align}
which satisfies $\sum_{s\in\cS}\sum_{a\in\cA}\zeta^k_n(s,a)=1,\forall n\in\cN$, and hence $\zeta^k_n,\forall n\in\cN$ is a probability measure.  Therefore,~(\ref{eq:lp-constraint1}) and~(\ref{eq:lp-constraint2}) are restatements of the ``long-term activation constraint'' and ``long-term fairness constraint'', respectively;~(\ref{eq:lp-constraint3}) represents the fluid transition of occupancy measure; and~(\ref{eq:lp-constraint4}) holds since the occupancy measure is a probability measure. 

\section{\fair index policy properties}\label{sec:appfairppt}
In this section, we discuss the properties of \fair index policy if we have full knowledge of transition and reward functions. In such case, the extended LP (\ref{eq:UCB_extended-obj}) to (\ref{eq:UCB_extended-cons6}) us equivalent to LP defined above (\ref{eq:lp-obj}) to (\ref{eq:lp-constraint4}) by letting $z_{n}^k(s,a, s^\prime) = P_n^k(s'|s,a)\zeta^k_{n}(s,a)$. (\ref{eq:UCB_extended-obj}) to (\ref{eq:UCB_extended-cons4}) are naturally satisfied and (\ref{eq:UCB_extended-cons5},\ref{eq:UCB_extended-cons6}) come from the construction of confidence ball. 

Due to the knowledge of true transition probabilities and reward functions, the optimal results of (\ref{eq:UCB_extended-obj}) to (\ref{eq:UCB_extended-cons6}) of any episode $k$ are the same, denoted by $\Omega_{\pi^*}=\{\zeta_n^*(s,a),\forall n\in\cN, s\in\cS, a\in\cA\}$, and the corresponding optimal value as $V^*:=\sum_{n=1}^N \sum_{s\in\cS}\sum_{a\in\cA} \zeta_{n}^*(s,a) r_n(s,a)$. The corresponding \emph{fair index} can be written as in~(\ref{eq:fair-index}). 

Next we rank all arms according to their indices in~(\ref{eq:fair-index}) in a non-increasing order, and activate the set of $B$ highest indexed arms, denoted as $\cN(t)\subset\cN$ such that $\sum_{n\in\cN(t)} a_n^*(t)\leq B$. All remaining arms are kept passive at time $t$. We denote the resultant index-based policy as $\pi^*=\{\pi_n^*, n\in\cN\}$, and call it the \fair index policy with full knowledge. 

\begin{algorithm}[t]
\caption{The Optimal \fair Index Algorithm for Offline \frmab}
\label{alg:offindex}
\begin{algorithmic}[1]
	\State {\bfseries Require:} Transition probability and reward function $P_n(s^\prime|s,a), r_n(s,a)$, $\forall n\in\cN, s,s^\prime\in\cS, a\in\cA$.	
	\State Solve (\ref{eq:lp-obj}) to (\ref{eq:lp-constraint4}) to achieve $\zeta^*_n(s,a)$;
	\State Compute optimal \fair index policy $\omega^*_n(s,a) = \frac{\zeta^*_n (s,1)}{\sum_{a}\zeta^*_n(s,a)}$
	\State Sort all \fair index $\omega^*_n(s,a)$, activate top $B$ arms with largest \fair index.
\end{algorithmic}
\end{algorithm}

\begin{remark}\label{remark:index}
Unlike Whittle-based policies \citep{whittle1988restless,hodge2015asymptotic,glazebrook2011general,zou2021minimizing,killian2021q}, our \fair index policy does not require the indexability condition, which is often hard to establish when the transition kernel of the underlying MDP is convoluted \citep{nino2007dynamic}. Like Whittle policies, our \fair index policy is computationally efficient since it is merely based on solving an LP.  A line of works \citep{hu2017asymptotically,zayas2019asymptotically,zhang2021restless,xiong2022reinforcement} designed index policies without indexability requirement for finite-horizon \rmab, and hence cannot be directly applied to our infinite-horizon average-cost formulation for \frmab in~(\ref{eq:FRMAB-obj})-(\ref{eq:FRMAB-constraint2}).  Finally, none of the aforementioned index policies guarantees fairness among arms, this is because these policies are only been well defined for maximizing the total reward~(\ref{eq:FRMAB-obj}) under the activation constraint~(\ref{eq:FRMAB-constraint1}), without taking the fairness constraint~(\ref{eq:FRMAB-constraint2}) into account. 
\end{remark}

\subsection{Asymptotic Optimality} 
We show that with perfect knowledge, the \fair index policy is asymptotically optimal in the same asymptotic regime as that in state-of-the-art \rmab literature \citep{whittle1988restless,weber1990index,verloop2016asymptotically}.  For abuse of notation, we denote the number of arms as $\rho N$, the activation constraint  as $\rho B$ in the asymptotic regime with $\rho\rightarrow \infty$.  In other words, we consider $N$ classes of arms and will be interested in this fluid-scaling process with parameter $\rho.$ Let $X_n^\rho(\pi^*, s, a, t)$ be the number of class-n arms at state $s$ taking action $a$ at time $t$ under the \fair index policy $\pi^*$.  Denote the long-term reward as $V_{\pi^*}^\rho=\liminf_{T\rightarrow\infty}\mathbb{E}_{\pi^*}\sum_{t=1}^T\sum_{n=1}^N\sum_{(s, a)}r_n(s,a)\frac{X_n^\rho(\pi^*, s, a, t)}{\rho}$.  Then our \fair index policy $\pi^*$ is asymptotically optimal if and only if $V_{\pi^*}^\rho\geq V_{\pi}^\rho, \forall \pi$.

\begin{definition}\label{def:global-attractor}
An equilibrium point $X^{\rho,*}/\rho$ under the \fair index policy $\pi^*$ is a global attractor for the process ${X^\rho(\pi^*;t)}/{\rho}$, if, for any initial point ${X^\rho(\pi^*;0)}/{\rho}$,
the process ${X^\rho(\pi^*;t)}/{\rho}$ converges to $X^{\rho,*}/\rho.$
\end{definition}

\begin{remark}\label{remark:global-attractor}
The global attractor indicates that all trajectories converge to $X^{\rho,*}$.  Though it may be difficult to establish analytically that a fixed point is a global attractor for the process \citep{verloop2016asymptotically}, such assumption has been widely made in \rmab literature \citep{weber1990index,hodge2015asymptotic,verloop2016asymptotically,zou2021minimizing,duran2018asymptotic}, and is only verified numerically.  Our experimental results in Section~\ref{sec:exp} show that such convergence indeed occurs for our \fair index policy $\pi^*$.
\end{remark}

\begin{theorem}\label{thm:asy-optimality}
 \fair index policy $\pi^*$ is asymptotically optimal under Definition~\ref{def:global-attractor}, i.e., $\lim_{\rho\rightarrow\infty} V_{\pi^*}^\rho-V_{\pi^{opt}}^\rho=0$. 
\end{theorem} 

\begin{lemma}\label{lemma:fair-constraint}
\fair index policy $\pi^*$ satisfies long term fairness constraint (\ref{eq:FRMAB-constraint2}) in asymptotic regime.
\end{lemma}

\subsection{Proof of Theorem~\ref{thm:asy-optimality}}

\begin{lemma}\label{lem:UpperBound}
The optimal value achieved by 
LP in $(\ref{eq:lp-obj})$-$(\ref{eq:lp-constraint4})$ is an upper bound of that achieved by \frmab in $(\ref{eq:FRMAB-obj})$-$(\ref{eq:FRMAB-constraint2})$. 
\end{lemma}

\begin{proof}
According to \citep{altman1999constrained}, the LP in $(\ref{eq:lp-obj})$-$(\ref{eq:lp-constraint4})$ is equivalent to the relaxed problem in $(\ref{eq:relaxed-FRMAB})$-$(\ref{eq:relaxed-FRMAB-constraint2})$. 
It is sufficient to show that the relaxed problem achieves no less average reward than that of the original problem in $(\ref{eq:FRMAB-obj})$-$(\ref{eq:FRMAB-constraint2})$.
The proof is straightforward since the constraints in the relaxed problem  expand the feasible region of the original problem \frmab in $(\ref{eq:FRMAB-obj})$-$(\ref{eq:FRMAB-constraint2})$. Denote the feasible region of the original problem  as
\begin{align*}
    \Gamma:=\Bigg\{a_n(t), \forall t\Bigg\vert\sum_{n=1}^{N}a_n(t)\leq B\Bigg\},
\end{align*}
and the feasible region of the relaxed problem as 
\begin{align*}
\Gamma^\prime:=\left\{a_n(t), \forall t\Bigg\vert\limsup_{T\rightarrow \infty} \frac{1}{T}\mathbb{E}\left\{\sum_{t=1}^T\sum_{n=1}^{N} a_n(t) \right\}\leq B\right\}.
\end{align*}
{It is clear that the relaxed problem expands the feasible region of original problem in $(\ref{eq:FRMAB-obj})$-$(\ref{eq:FRMAB-constraint2})$, i.e., $\Gamma\subseteq\Gamma^\prime.$  Therefore, the relaxed problem  achieves an objective value no less than that of the original problem in $(\ref{eq:FRMAB-obj})$-$(\ref{eq:FRMAB-constraint2})$ because the original optimal solution is also inside the relaxed feasibility set. This indicates the LP in $(\ref{eq:lp-obj})$-$(\ref{eq:lp-constraint4})$ achieves an optimal value no less than that of $(\ref{eq:FRMAB-obj})$-$(\ref{eq:FRMAB-constraint2})$.}
\end{proof}

 We redefine  the expected long-term average reward with scaling parameter $\rho$ of \fair index policy $\pi^*$ in a continuous-time domain as 
\begin{align*}
    &V_{\pi^*}^\rho:=\liminf_{T\rightarrow \infty} \frac{1}{T} \mathbb{E}_{\pi^*} \nonumber \\
    & \qquad \left(\int_{t=1}^T\sum_{n=1}^{N}\sum_{(s,a)} r_n(s,a)\frac{X_n^\rho(\pi^*,s,a;t)}{\rho}dt\right).
\end{align*}
The key of this proof relies on showing that {the fluid process $\frac{X^\rho_n(\pi^*;t)}{\rho}$} converges to $\{\zeta_n^*, \forall n\}$ under the proposed \fair index policy $\pi^*$ when $\rho\rightarrow\infty$. Since the $\{\zeta_n^*, \forall n\}$ is an optimal solution of the LP $(\ref{eq:lp-obj})$-$(\ref{eq:lp-constraint4})$,  according to Lemma \ref{lem:UpperBound}, the proposed \fair index policy $\pi^*$ achieves no worse average reward compared to the optimal policy $\pi^{opt}$, i.e., $\lim_{\rho\rightarrow\infty} V_{\pi^*}^\rho-V_{\pi^{opt}}^\rho\geq 0.$ One the other hand, it is always true that 
 $\lim_{\rho\rightarrow\infty} V_{\pi^*}^\rho-V_{\pi^{opt}}^\rho\leq 0$ by the definition of $\pi^{opt}.$ This will give rise to the desired result. 
To prove Theorem 1, we first introduce some auxiliary definition and lemmas. 
\begin{definition}[Density dependent population process \citep{gast2010mean}]\label{def:DDPP}
A sequence of Markov process $X^\rho$ on $\frac{1}{\rho}\mathbb{N}^d (d\geq 1)$ is called a density dependent population process if there exists a finite number of transitions, say $\mathcal{L}\subset \mathbb{N}^d$, such that for each $\ell\in\mathcal{L}$, the rate of transition from $X^\rho$ to $X^\rho+\ell/\rho$ is $\rho f_\ell(X^\rho)$, where $f_\ell(\cdot)$ does not depend on $\rho$.
\end{definition}

To show the convergence of the density dependent population process, we consider $F$ the function $F(x)=\sum_{\ell\in\mathcal{L}}\ell f_\ell(x)$ and the following ordinary differential equation $x(0)=x_0$ and $\dot{x}_{x_0}(t)=F(x_{x_0}(t))$. The following lemma shows that the stochastic process $X^{\rho}(t)$ converges to the deterministic $x(t)$.

\begin{lemma}\citep{gast2010mean}
Assume for all compact set $E\subset \mathbb{R}^d, \sum_{\ell}|\ell|\sup_x f_{\ell}(x)<\infty,$ and $F$ is Lipschitz on $E$. If $\lim_{\rho\rightarrow\infty} X^\rho(0)=x_0$ in probability, then for all $t>0$:
\begin{align*}
    \lim_{\rho\rightarrow\infty} \sup_{s\leq t}|X^\rho(s)-x(s)|=0,~ \text{in probability}.
\end{align*}

\end{lemma}

The following lemma shows that under the global attractor property of function $F$, the stationary distribution of the stochastic density population process converges to the dirac measure of the global attractor. 
\begin{lemma}\label{lemma_dirac}
\citep{gast2010mean} If $F$ has a unique stationary point $x^*$ to which all trajectories converge, then the stationary measures $\zeta^\rho$ concentrate around $x^*$ as $\rho$ goes to infinity:
\begin{align*}
    \lim_{\rho\rightarrow\infty} \zeta^\rho\rightarrow \delta_{x^*},
\end{align*}
where $\delta_{x^*}$ is the dirac measure in $x^*$.
\end{lemma}

Provided these lemmas, we are now ready to prove Theorem~\ref{thm:asy-optimality}. 
\begin{proof}
{ We denote $A_n^{\pi^*}(s)$ as the set of all combinations $(m,j), m\in\cN, j\in\cS$ such that class-$m$ arms in state $j$ have larger indices than those of class-$n$ arms in state $s$ under the \fair policy $\pi^*$.}  The transition rates of the process $X^\rho(\pi^*;t)/\rho$ are then defined as
\begin{align}\label{eq:fluid_transition}
     x\rightarrow x-\frac{e_{n,s}}{\rho}+\frac{e_{n,s^\prime}}{\rho}~\text{at rate} \sum_{a}{P_n(s,a,s^\prime)}x_n^\rho(s,a), 
\end{align}
where $\sum_{a\in\cA\setminus\{0\}}ax_n^\rho(s,a)=\min\left(\rho B-\sum_{(m,j)\in A_n^{\pi^*}(s)}\sum_{a\in\cA\setminus\{0\}}ax^\rho_m(j,a), 0\right),$ {and $e_{n,s}\in\mathbb{R}^{S\times 1}$ is the unit vector with the $s$-th position being $1$.}

It follows that there exists a continuous function $ f_\ell(x)$ to model the transition rate of the process $X^\rho(\pi^*;t)$ from state $x$ to $x+\ell/\rho, \forall \ell\in \mathcal{L}$ according to \eqref{eq:fluid_transition}, with $\mathcal{L}$ being the set composed of a finite number of vectors in $\mathbb{N}^{SN}$. Hence, the process $X^\rho(\pi^*;t)/\rho$ is a density dependent population processes according to Definition \ref{def:DDPP}.

Note that the process $X^\rho(\pi^*;t)$
can be expressed 
\begin{align*}
    \frac{d X^\rho(\pi^*;t)}{dt}=F(X^\rho(\pi^\star;t)),
\end{align*}
with $F(\cdot)$ being Lipschitz continuous and satisfying $F(X^\rho(\pi^*;t))=\sum_{\ell\in\mathcal{L}}\ell f_\ell(X^\rho(\pi^*;t)).$  Under the condition that the considered MDP is unichain, such that the process $\frac{X^\rho(\pi^*;t)}{\rho}$ has a unique invariant probability distribution $\zeta^\rho_{\pi^*},$ which is tight \citep{verloop2016asymptotically}. Thus, we have $\zeta^\rho_{\pi^*}\left(\frac{X^\rho(\pi^*;t)}{\rho}\right)$ converge to the Dirac measure in $X^{\rho, *}/\rho$ when $\rho\rightarrow\infty$, which is a global attractor of $\frac{X^\rho(\pi^*;t)}{\rho}$ according to Lemma \ref{lemma_dirac}.
Therefore, according to the ergodicity theorem \citep{cinlar1975introduction}, we have
\begin{align*}
    \lim_{\rho\rightarrow\infty} V_{\pi^*}^\rho&=  \lim_{\rho\rightarrow\infty}\sum_{n=1}^{N}\sum_{(s,a)\in\cS\times\mathcal{A}}{\sum_{\frac{X^\rho(\pi^*, s,a)}{\rho}\in\mathcal{X}}}\\
    &\qquad \zeta_{\pi^*}^\rho\left(\frac{X_n^\rho(\pi^*,s,a)}{\rho}\right) r_n(s,a)\frac{X_n^\rho(\pi^*,s,a)}{\rho}\displaybreak[0]\\
    &=\sum_{n=1}^{N}\sum_{(s,a)\in\cS\times\cA} r_n(s,a)\omega^*_n(s,a)\\
    &\geq\lim_{\rho\rightarrow\infty} V_{\pi^{opt}}^\rho,
\end{align*}
where the second equality is due to the fact that $\zeta_{\pi^*}^\rho\left(\frac{X_n^\rho(\pi^*,s,a)}{\rho}\right)$ converges to the Dirac measure in $X^{\rho, *}/\rho$ when $\rho\rightarrow\infty$ under the global attractor condition. 

\end{proof}

\subsection{Proof of Lemma~\ref{lemma:fair-constraint}}
Recall in the proof of Theorem~\ref{thm:asy-optimality}, we have the following results: 
$\frac{X^\rho_n(\pi^*;s,a,t)}{\rho}$ converges to $\{\zeta_n^*, \forall n\}$ under the proposed \fair index policy $\pi^*$ when $\rho\rightarrow\infty$, where $X^\rho_n(\pi^*;s,a,t)$ is the number of class $n$ arms at state $s$ taking action $a$ at time $t$ under \fair $\pi^*$. Define $a_{n,j}(t), j\in\{1,2,...,\rho\}$ as the action of $j$th arm in class $n$ at time $t$ under \fair $\pi^*$,
\begin{align*}
 \frac{ \sum_{j=1}^\rho a_{n,j}(t)}{\rho} &= \frac{ \sum_{j=1}^\rho a_{n,j}(t)| a_{n,j}(t)=1}{\rho}\\
 &= \frac{\sum_s X^\rho_n(\pi^*;s,1,t)}{\rho}\\
 &= \sum_s \zeta_n^*(s,1), \forall t.
\end{align*}
 By taking average over $T$ we have
\begin{align*}
     \frac{1}{T} \sum_{t=1}^T \frac{ \sum_{j=1}^\rho a_{n,j}(t)}{\rho} &= \frac{1}{T} \sum_{t=1}^T \frac{\sum_s X^\rho_n(\pi^*;s,a,t)}{\rho} \\
     &= \sum_s \zeta_n^*(s,1). 
\end{align*}
According to constraint \ref{eq:lp-constraint2}, $\sum_s \zeta_n^*(s,1) \geq \eta_n, \forall n$, 
\begin{align*}
    \frac{1}{T} \sum_{t=1}^T \frac{ \sum_{j=1}^\rho a_{n,j}(t)}{\rho} \geq \eta_n, \forall n.
\end{align*}
When $\rho, T \rightarrow\infty $,
\begin{align*}
    \liminf_{T \rightarrow \infty} \frac{1}{T} \mathbb{E}\Bigg[\sum_{t=1}^T a_n(t)\Bigg]\geq\eta_n, \forall n,
\end{align*}
which shows Lemma~\ref{lemma:fair-constraint} holds in asymptotic regime.

\begin{figure*}[t]
 \centering
    \begin{minipage}{.48\textwidth}
 \centering
 \includegraphics[width=1\columnwidth]{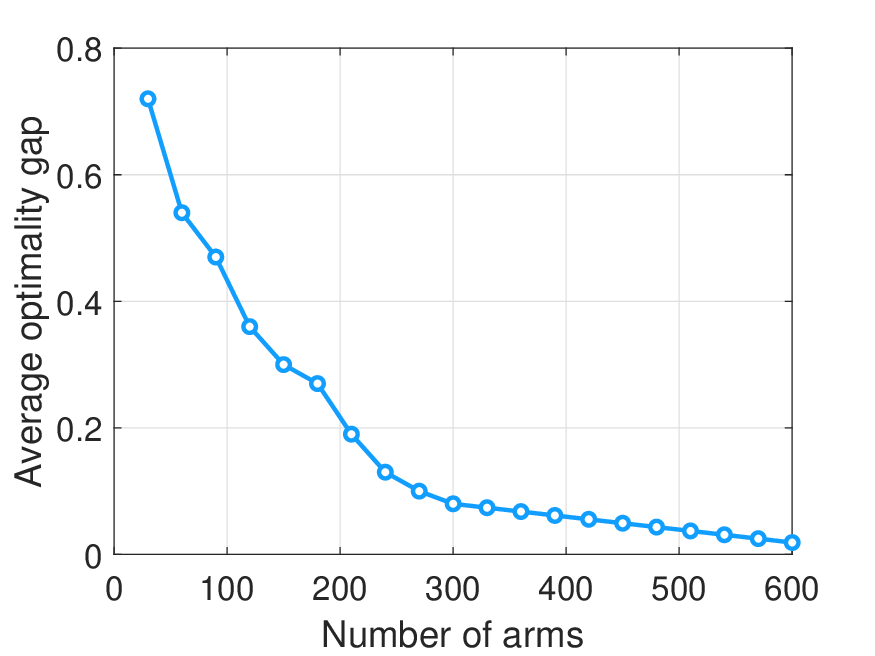}
 \vspace{-0.2in}
\caption{Asymptotic optimality of \fair index policy.}
\label{fig:asymptotic}
 \end{minipage}\hfill
 \begin{minipage}{.48\textwidth}
 \centering
 \includegraphics[width=1\columnwidth]{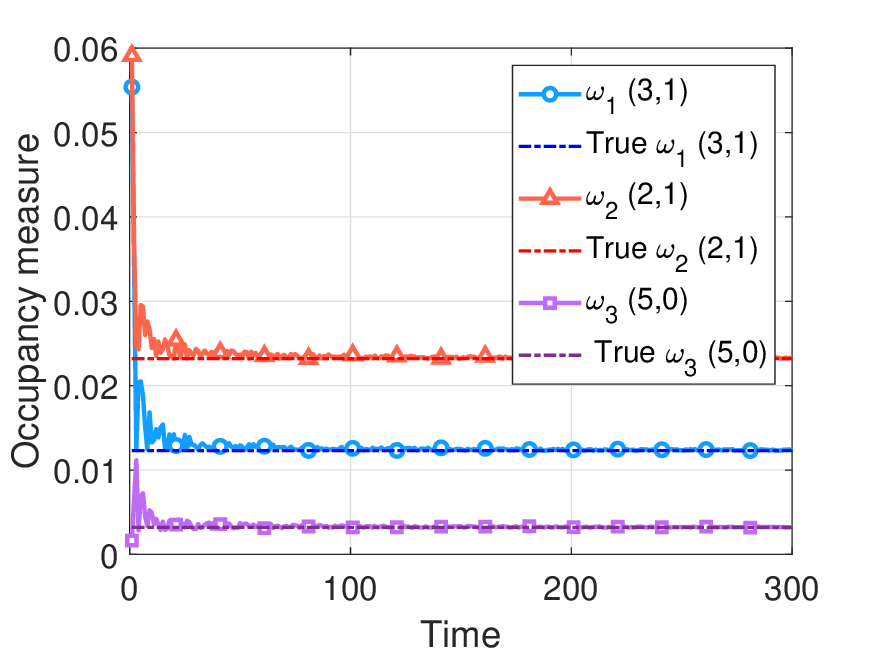}
 \vspace{-0.2in}
 \caption{Global attractor for \fair index policy.}
\label{fig:global-attractor}
 \end{minipage}\hfill
 \vspace{-0.1in}
 \end{figure*}

\subsection{Experiments of asymptotic optimality based on synthetic trace}
Using the synthetic trace setting in Section~\ref{sec:exp}, we verify the asymptotic optimality and global attractor assumption. We vary the number of arms from 30 to 600, and keep the budget as $\frac{1}{3}$ of the number of arms.

\textbf{Asymptotic Optimality.} We first validate the asymptotic optimality of \fair index policy. In particular, we define the difference of average reward obtained by \fair index policy with that from the theoretical upper bound solved from the LP~(\ref{eq:lp-obj})-(\ref{eq:lp-constraint4}).  We call the ratio between this award difference and the number of arms as \textit{optimality gap}.  From Figure~\ref{fig:asymptotic}, we observe that as the number of arms increases, the optimality gap decreases significantly and closes to zero.  This verifies the asymptotic optimality in  Theorem~\ref{thm:asy-optimality}.

\textbf{Global Attractor.} As indicated in Remark~\ref{remark:global-attractor}, the asymptotic optimality of our \fair index policy is under the definition of global attractor as in state of the arts.  In Figure~\ref{fig:global-attractor}, we randomly pick three state-action pairs $(3,1), (2,1)$ and $(5,0)$ for illustration.  It is clear that the occupancy measure of arm 1 for state-action pair $(3,1)$ indeed converges.  Similarly for arm 2 with state-action pair $(2,1)$ and for arm 3 with state-action pair $(5,0)$.  Therefore, the convergence indeed occurs for our \fair index policy and hence we verify the global attractor condition.

\section{Additional Experimental Details}

\subsection{Continuous Positive Airway Pressure Therapy (CPAP)} 
We study the CPAP as in \citep{kang2013markov,herlihy2023planning,li2022towards}, which is a highly effective treatment when it is used consistently during the sleeping for adults with obstructive sleep apnea.  
The state space is ${1,2,3}$, which represents low, intermediate and acceptable adherence level respectively. Based on their adherence behaviour, patients are clustered into two groups, with the first group named  ``Adherence'' and the second ''Non-Adherence''. The difference between these two groups reflects on their transition probabilities, as in Figures.  \ref{Fig:CPAPtransition1}- \ref{Fig:CPAPtransition2}. Generally speaking, the first group has a higher probability of staying in a good adherence level. From each group, we construct 10 arms , whose transition probability matrices are generated by adding a small noise to the original one. Actions such as text patients/ making a call/ visit in person will cause a 5\% to 50\% increase in adherence level. The budget is set to $B = 5$ and  the fairness constraint is set to be a random number between [0.1, 0.7].  The objective is to maximize the total adherence level.

\subsection{PASCAL Recognizing Textual Entailment task (PASCAL-RTE1)} 
This is a task aims to infer hypothesis as in \citep{snow2008cheap}. In each question, the annotator is provided with two sentences and tasked with making a binary choice regarding whether the second hypothesis sentence can be inferred from the first. For example, ``Egyptian television will make a series about Moslems, Copts and Boutros Boutros Ghali" can be inferred from ``Egyptian television is preparing to film a series that highlights the unity and cohesion of Moslems and Copts as the single fabric of the Egyptian society, exemplifying in particular the story of former United Nations Secretary-General Boutros Ghali". Due to carelessness or lack of background information, the annotator may not label correctly.  The 10 workers we pick have worker id (A11GX90QFWDLMM, A14WWG6NKBDWGP, A151VN1BOY29J1, A15MN5MDG4D7Q9, A19PMUTQXDIPLZ, A1CP0KZJS5LSIF, A1DCEOFAUIDY58, A1IPO1FAD1A60X, A1LY3NJTYW9TFF, A1M0SEWUJYX9K0). The ``successful annotation probability'' for each worker is averaged from all tasks, which is (0.495, 0.45, 0.4, 0.3, 0.6, 0.55, 0.65, 0.5, 0.54, 0.37). The state space is $\{0,1\}$, action space is $\{0,1\}$. The number of arms is $N=10$ and budget is set to $B=3$. We set $\eta = 0.05$ for all arms.

\subsection{Land Mobile Satellite System (LMSS)} 
We consider a similar setting as in \citep{wang2020restless} but extend it to the setting with more arms. The land mobile satellite needs to point at different directions (elevation angles). Each elevation angle will result in different channel parameters, hence different transition probabilities.  We pick communicating via S-band in the urban environment and the corresponding parameters of two state Markov chain representations on the channel model. State 1 is defined as \emph{Good} state and state 2 as \emph{Bad}. There are 4 arms in total, which represents $40^{\circ}, 60^{\circ}, 70^{\circ}, 80^{\circ}$ elevation angle. The transition probabilities can be found in Table \ref{table:LMSS}, which can be found in Table III \citep{https://doi.org/10.1002/sat.964}. The state space is $\{0,1\}$, action space is $\{0,1\}$. The number of arms is $N=4$ and budget is set to $B=2$. We set $\eta = 0.03$ for all arms.
\begin{table}[t]
\begin{center}
\begin{tabular}{ |c||c|c|c|c|  }
 \hline
 Elevation angle & $p_{1,1}$ & $p_{1,2}$ & $p_{2,1}$ & $p_{2,2}$\\
 \hline
 $40^{\circ}$   & 0.9155 & 0.0845   &   0.0811 &0.9189\\
 $60^{\circ}$   &  0.9043 &  0.0957   &0.2 &0.8\\
 $70^{\circ}$   & 0.9155 & 0.0845   & 0.2069 & 0.7931\\
 $80^{\circ}$   & 0.9268 & 0.0732    &  0.2667 & 0.7333\\
 \hline
\end{tabular}
\end{center}
\caption{Transition probability matrix of LMSS}
\label{table:LMSS}
\end{table}

\begin{figure}[ht]
\begin {center}
\begin {tikzpicture}[-latex ,auto ,node distance =4 cm and 5cm ,on grid ,
semithick ,
state/.style ={ circle ,top color =white , bottom color = processblue!20 ,
draw,processblue , text=blue , minimum width =1 cm}] x
\node[state] (C)
{$3$};
\node[state] (A) [above left=of C] {$1$};
\node[state] (B) [above right =of C] {$2$};
\path (A) edge [loop above] node[above] {$0.0385$} (A);
\path (C) edge [bend left =25] node[below =0.15 cm] {$0.0257$} (A);
\path (A) edge [bend right = -15] node[below =0.15 cm] {$0.9615$} (C);
\path (C) edge [loop left] node[left] {$0.9498$} (C);
\path (C) edge [bend left =15] node[below =0.15 cm] {$0.0245$} (B);
\path (B) edge [bend right = -25] node[below =0.15 cm] {$1$} (C);
\end{tikzpicture}
\caption{Transition diagram for CPAP Cluster 1}
\label{Fig:CPAPtransition1}
\end{center}
\end{figure}

\begin{figure}[ht]
\begin {center}
\begin {tikzpicture}[-latex ,auto ,node distance =4 cm and 5cm ,on grid ,
semithick ,
state/.style ={ circle ,top color =white , bottom color = processblue!20 ,
draw,processblue , text=blue , minimum width =1 cm}] x
\node[state] (C)
{$3$};
\node[state] (A) [above left=of C] {$1$};
\node[state] (B) [above right =of C] {$2$};
\path (A) edge [loop above] node[above] {$0.7427$} (A);
\path (C) edge [bend left =25] node[below =0.15 cm] {$0.2323$} (A);
\path (A) edge [bend right = -15] node[below =0.15 cm] {$0.1835$} (C);
\path (A) edge [bend right = -15] node[below =0.15 cm] {$0.0741$} (B);
\path (C) edge [loop left] node[left] {$0.6657$} (C);
\path (C) edge [bend left =15] node[below =0.15 cm] {$0.1020$} (B);
\path (B) edge [bend right = -25] node[below =0.15 cm] {$0.4967$} (C);
\path (B) edge [bend left = 15] node[below =0.15 cm] {$0.3399$} (A);
\path (B) edge [loop above] node[above] {$0.1634$} (B);
\end{tikzpicture}
\caption{Transition diagram for CPAP Cluster 2}
\label{Fig:CPAPtransition2}
\end{center}
\end{figure}

\vfill

\end{document}